\documentclass{article}

\PassOptionsToPackage{numbers}{natbib}
\usepackage[final]{neurips_2021}

\usepackage{graphicx}
\usepackage{subfigure}
\usepackage{mathtools}
\usepackage{amssymb}
\usepackage{amsthm}
\usepackage{thmtools,thm-restate}
\usepackage{multirow}
\usepackage{multicol}
\usepackage[algo2e,ruled,vlined,linesnumbered]{algorithm2e}
\usepackage{bbding}
\usepackage{amssymb}
\usepackage{wrapfig}

\SetKwRepeat{Do}{do}{while}
\newtheorem{theorem}{Theorem}[section]

\usepackage{color}

\newcommand{\defeq}{\vcentcolon=}

\makeatletter
\newcommand{\printfnsymbol}[1]{%
  \textsuperscript{\@fnsymbol{#1}}%
}

\usepackage[utf8]{inputenc} 
\usepackage[T1]{fontenc}    
\usepackage{hyperref}       
\usepackage{url}            
\usepackage{booktabs}       
\usepackage{amsfonts}       
\usepackage{nicefrac}       
\usepackage{microtype}      
\usepackage{xcolor}         

\title{On Effective Scheduling of Model-based Reinforcement Learning}

\author{%

Hang Lai\thanks{equal contribution. $^\dagger$Weinan Zhang is the corresponding author.}~~$^1$, Jian Shen$^\ast$$^1$, Weinan Zhang$^\dagger$$^1$, Yimin Huang$^2$,\\
\textbf{Xing Zhang}$^2$, \textbf{Ruiming Tang}$^2$, \textbf{Yong Yu}$^1$, \textbf{Zhenguo Li}$^2$
\\
$^1$Shanghai Jiao Tong University, $^2$Huawei Noah's Ark Lab\\
\texttt{\{laihang, wnzhang\}@apex.sjtu.edu.cn} \\
}

\begin{document}

\maketitle

\begin{abstract}
Model-based reinforcement learning has attracted wide attention due to its superior sample efficiency. Despite its impressive success so far, it is still unclear how to appropriately schedule the important hyperparameters to achieve adequate performance, such as the real data ratio for policy optimization in Dyna-style model-based algorithms. In this paper, we first theoretically analyze the role of real data in policy training, which suggests that gradually increasing the ratio of real data yields better performance. Inspired by the analysis, we propose a framework named AutoMBPO to automatically schedule the real data ratio as well as other hyperparameters in training model-based policy optimization (MBPO) algorithm, a representative running case of model-based methods. On several continuous control tasks, the MBPO instance trained with hyperparameters scheduled by AutoMBPO can significantly surpass the original one, and the real data ratio schedule found by AutoMBPO shows consistency with our theoretical analysis.


\end{abstract}

\section{Introduction}
Deep model-free reinforcement learning (MFRL) has achieved great successes in complex decision-making problems such as Go \cite{alphago} and robotic control \cite{gu2017deep}, to name a few. Although MFRL can achieve high asymptotic performance, a tremendous number of samples collected in interactions with the environment are required. In contrast, model-based reinforcement learning (MBRL), which alternately learns a model of the environment and derives an agent with the help of current model estimation, is considered to be more sample efficient than MFRL \cite{pmlr-v99-sun19a}.

MBRL methods generally fall into different categories according to the specific usage of the learned model \cite{ampo, zhu2020bridging}. Among them, Dyna-style algorithms \cite{dyna} adopt some off-the-shelf MFRL methods to train a policy with both real data from environment and imaginary data generated by the model. Since Dyna-style algorithms can seamlessly take advantage of innovations in MFRL literature and have recently shown impressive performance \cite{mbpo, benchmarking}, this paper mainly focuses on Dyna-style algorithms. Although some theoretical analysis \cite{slbo, mbpo,rajeswaran2020game} and thorough empirical evaluation \cite{benchmarking} have been conducted in the previous literature, it remains unclear how to appropriately schedule the hyperparameters to achieve optimum performance when training a Dyna-style MBRL algorithm. In practice, many important hyperparameters may primarily affect performance. 

Firstly, since Dyna-style MBRL algorithms consist of model learning and policy optimization, how to balance the alternate optimization of these two parts is a crucial problem \cite{rajeswaran2020game}. Intuitively, insufficient model training may lead to inaccurate estimation while excessive model training will cause overfitting, and a small amount of policy training may not utilize the model adequately while exhaustive policy training will exploit the model's deficiencies. Moreover, the rollout length of the imaginary trajectory is also critical \cite{mbpo} since too short rollout length fails to sufficiently leverage the model to plan forward, while too long rollout length may bring disastrous compounding error \cite{asadi2018lipschitz}. \citet{mbpo} manually design a schedule that linearly increases the rollout length across epochs. Finally, when using both real samples and imaginary samples to train the policy, how to control the ratio of the two datasets remains unclear. \citet{mbpo} fix the ratio of real data as $5\%$, while \citet{kalweit2017uncertainty} use the uncertainty to adaptively choose the ratio, which tends to use more imaginary samples initially and gradually use more real samples afterward. According to these existing works, the optimal hyperparameters in model-based methods may be dynamic during the whole training process, which further strengthens the burden of manually scheduling the hyperparameters.

Based on these considerations, in this work, we aim to investigate how to appropriately schedule these hyperparameters, i.e., real data ratio, model training frequency, policy training iteration, and rollout length, to achieve optimal performance of Dyna-style MBRL algorithms. Although real data ratio is an essential factor empirically \cite{kalweit2017uncertainty,mbpo}, it has not yet been studied thoroughly in theory. To bridge this gap, we first derive a return discrepancy upper bound for model-based value iteration, which reveals that \emph{gradually increasing the ratio of real data yields a tighter bound than choosing a fixed value}. Inspired by the analysis, considering the complex interplay between real data ratio and other hyperparameters in practice, we develop AutoMBPO to automatically determine the joint schedule of the above hyperparameters in model-based policy optimization (MBPO), a representative running case of Dyna-style MBRL algorithms. Specifically, AutoMBPO introduces a parametric hyper-controller to sequentially choose the value of hyperparameters in the whole optimization process to maximize the performance of MBPO. We apply AutoMBPO to several continuous control tasks. On all these tasks, the MBPO instance trained with hyperparameters scheduled by AutoMBPO can significantly surpass the one with original configuration \cite{mbpo}. Furthermore, the hyperparameter schedule found by AutoMBPO is consistent with our theoretical analysis, which yields a better understanding of MBPO and provides insights for the design of other MBRL methods.


\section{Related Work}

Model-based reinforcement learning (MBRL) has been widely studied due to its high sample efficiency. MBRL methods can be roughly categorized into four types according to different model usage: (i) Dyna-style algorithms \cite{dyna, slbo, clavera2018model, metrpo, mbpo, bmpo} leverage the model to generate imaginary data and adopt some off-the-shelf MFRL algorithms to train a policy using both real data and imaginary data; (ii) shooting algorithms \cite{nagabandi2018neural,pets} use model predictive control (MPC) to plan directly without explicit policy; (iii) analytic-gradient algorithms \cite{pilco,gps,clavera2020model} search policies with back-propagation through time by exploiting the model derivatives; (iv) model-augmented value expansion algorithms \cite{feinberg2018model, buckman2018sample} utilize model rollouts to improve the target value for temporal difference (TD) updates. This paper mainly focuses on the first category.

Dyna-style algorithms \cite{dyna} typically alternate between sampling real data, learning dynamics models, generating imaginary data, and optimizing policy. Many efforts have been made to improve one or more of these steps. For model learning, the deep ensemble technique \cite{rajeswaran2016epopt, pets, metrpo} has been leveraged to resist overfitting. As for imaginary data generation, it is suggested to generate short rollouts to reduce compounding model error \cite{mbpo}, and the model's uncertainty is further incorporated to choose reliable imaginary data dynamically \cite{m2ac}. When optimizing the policy, we can use merely imaginary data \cite{slbo,metrpo} or a mixture of real data and imaginary one in a fixed ratio \cite{mbpo}. Moreover, the ratio can also be dynamically adjusted according to the estimated uncertainty of the Q-function \cite{kalweit2017uncertainty}. Besides, \citet{rajeswaran2020game} present a game-theoretic framework for MBRL by formulating the optimization of model and policy as a two-player game.

As for automatic hyperparameter optimization for MBRL, \citet{zhang2021importance} utilize population based training (PBT) to tune the hyperparameters of a shooting algorithm dynamically. The reinforcement on reinforcement (RoR) framework proposed by \citet{dong2020intelligent} applies an RL algorithm to control the sampling and training process of an MBRL method. Our work differs from theirs in four major aspects: (1) We provide a theoretical analysis for the hyperparameter being scheduled, which is not included in their work. (2) In terms of problem settings, RoR assumes one can directly access an initial state without further interaction from that state, which is impractical for most RL problems. (3) As for the method design, our formulated hyper-MDP in Section \ref{sec:Hyper-MDP Formulation} is totally different from RoR's, including state, action and reward. In detail, our state definition contains more training information, and the reward design is based on the return rather than a single-step reward in the target-MDP. (4) Finally, in the experiments in Section \ref{sec: exp}, our method achieves better performance and generalization in all environments, and the empirical findings also show differences.

\section{Preliminaries}
We first briefly introduce the RL problem with notations used throughout this paper and the previous MBRL method on which our proposed framework is based.
\subsection{Reinforcement Learning}
In reinforcement learning (RL), there originally exists a Markov decision process (MDP), which we will call \emph{target-MDP} throughout the paper. The target-MDP is defined by the tuple ($\mathcal{S}, \mathcal{A}, T, r, \gamma$). $\mathcal{S}$ and $\mathcal{A}$ are the state and action spaces, respectively, and $T(s'\mid s,a)$ is the transition density of state $s'$ given action $a$ taken under state $s$. The reward function is denoted as $r(s,a)$, and $\gamma \in (0,1)$ is the discount factor. The goal of RL is to find the optimal policy $\pi^*$ that maximizes the expected return, denoted by $\eta$:
\begin{equation}
\label{eq: rl-obj}\small
\pi^* \defeq \mathop{\arg \max}_\pi \eta[\pi]=\mathop{\arg \max}_\pi \mathbb{E}_\pi \Big[\sum_{t=0}^\infty \gamma^t r(s_t,a_t) \Big],
\end{equation}
where $s_{t+1} \sim T(s\mid s_t,a_t)$ and $a_t \sim \pi(a\mid s_t)$. Generally, the ground truth transition $T$ is unknown, and MBRL methods aim to construct a model $\hat{T}$ of the transition dynamics to help improve the policy.

To evaluate a reinforcement learning method (no matter whether it is model-free or model-based), we mainly care about two metrics: (1) asymptotic performance: the expected return of the converged policy in real environments; (2) sample efficiency: the number of samples $(s,a,s^{\prime},r)$ collected in target-MDP to achieve some performance. Therefore, an effective RL method means it achieves higher expected returns with fewer real samples. 

\subsection{Model-Based Policy Optimization}
Model-based policy optimization (MBPO) \cite{mbpo} is a state-of-the-art MBRL method which will be used as the running case in our framework. MBPO learns a bootstrapped ensemble of probabilistic dynamics models 
$\{\hat{T}_\theta^1, ..., \hat{T}_\theta^E \}$
where $E$ is the ensemble size. Each individual dynamics model in the ensemble is a probabilistic neural network that outputs a Gaussian distribution with diagonal covariance $\hat{T}_{\theta}^{i}\left(s^{\prime} | s, a\right)=\mathcal{N}\left({\mu}_{\theta}^{i}\left({s}, {a}\right), {\Sigma}_{\theta}^{i}\left({s}, {a}\right)\right)$. These models are trained independently via maximum likelihood with different initializations and training batches. The corresponding loss function is
\begin{align} 
\mathcal{L}_{\hat{T}}({\theta})=&\sum_{n=1}^{N}\left[\mu_{\theta}\left({s}_{n}, {a}_{n}\right)-{s}_{n+1}\right]^{\top} {\Sigma}_{\theta}^{-1}\left({s}_{n}, {a}_{n}\right) \label{equation:forward model loss}
 \left[\mu_{\theta}\left({s}_{n}, {a}_{n}\right)-{s}_{n+1}\right]+\log \operatorname{det} {\Sigma}_{\theta}\left({s}_{n}, {a}_{n}\right). \nonumber
\end{align}
In practical implementation, an early stopping trick is adopted in training the model ensemble. To be more specific, when training each individual model, a hold-out dataset will be created, and the training will early stop if the loss evaluated on the hold-out data does not decrease.

Every time the ensemble models are trained, they are then used to generate $k$-length imaginary rollouts, which begin from states sampled from real data and follow the current policy. Each time the policy interacts with the target-MDP, it will be trained for $G$ gradient updates using both real data and imaginary data in a certain ratio. To make it clear, in each batch for policy training, we denote the proportion of real data to the whole batch as $\beta$ and the proportion of imaginary data as $1-\beta$. We call $\beta$ the \textbf{\emph{real ratio}} and call $G$ the \textbf{\emph{policy training iteration}} throughout the paper. The policy optimization algorithm is SAC (Soft Actor-Critic) \cite{sac}, and the loss functions for the actor $\pi_\omega$ and the critic $Q_\phi$ are
\begin{align}
\mathcal{L}_\pi(\omega) =~ & \mathbb{E}_{s_t \sim D} \Big[\mathbb{E}_{a_t \sim \pi_\omega} [\alpha \log(\pi_\omega(a_t|s_t)) - Q_\phi(s_t, a_t)] \Big],\\
\mathcal{L}_Q(\phi) =~ & \mathbb{E}_{(s_t,a_t)\sim D}\Big[\frac{1}{2} \big(Q_\phi(s_t, a_t) - (r(s_t,a_t) \\
& + \gamma \mathbb{E}_{s_{t+1}, a_{t+1}}[Q_{\hat{\phi}}(s_{t+1}, a_{t+1})-\alpha \log \pi_\omega (a_{t+1}|s_{t+1})])\big)^{2}\Big],\nonumber
\end{align}
where $Q_{\hat{\phi}}$ is the target Q-function.

\section{Analysis of real ratio schedule}
\label{sec:theory}
Previous theoretical findings in MBRL mainly focus on the optimization of model and policy \cite{slbo,rajeswaran2020game} or the rollout length \cite{mbpo,bmpo}, while the usage of real data in policy optimization remains unclear. \citet{kalweit2017uncertainty} only use real data near convergence but accept noisier imaginary data initially, which seems confusing since the model may not be accurate enough at the beginning. In this section, we provide a theoretical analysis of the real ratio schedule by deriving a return discrepancy upper bound of sampling-based Fitted Value Iteration (FVI) in model-based settings. The analysis mainly follows \citet{munos2008finite}, where all the training data are sampled from the underlying dynamics. We extend the framework to a more general case: the data can either be sampled from the underlying dynamics w.p. (with probability) $\beta$ or from a learned model w.p. $1-\beta$, which we call \emph{$\beta$-mixture sampling-based FVI}. Besides, we consider deterministic environments here, i.e., given $(s,a)$, the next state $s^{\prime} = T(s,a)$ is unique, just as in all the environments in Section~\ref{sec: exp}.

For the sake of specificity, we first present a detailed description of the $\beta$-mixture sampling-based FVI algorithm. Let $\mathcal{F}$ be the value function space, and $V_0 \in \mathcal{F}$ be the initial value function. The FVI algorithm produces a sequence of functions $\{V_k\}_{0 \leq k \leq K} \subset \mathcal{F}$ iteratively. Given $V_k$ and $N$ sampled states $\{s_1, ..., s_N \}$, the value function of the next iteration $V_{k+1}$ is computed as
\begin{align}
    \hat{V}_{k}(s_i)&=\max _{a \in \mathcal{A}} \big(r_i+\gamma V_k\left(s_i^{\prime}\right)\big), i = 1,2,\ldots,N,
    \\
    V_{k+1}&=\underset{f \in \mathcal{F}}{\operatorname{argmin}} \sum_{i=1}^{N}\left|f\left(s_{i}\right)- \hat{V}_k(s_i) \right|^{p},
\end{align}
where $s_i$ is uniformly sampled from a certain state distribution $\mu$, and for each possible action $a$, the next state $s_{i}^{\prime}$ is sampled from the underlying dynamics w.p. $\beta$ or from the learned model w.p. $1- \beta$. As final preparatory steps, we define the 
\emph{Bellman operator} for deterministic MDPs $B:\mathcal{F}\rightarrow \mathcal{F}$ as
\begin{equation}
    \nonumber
    B V(s) =\max _{a \in \mathcal{A}}\left\{r(s, a)+\gamma V(s^{\prime})\right\},
\end{equation}
where $s^{\prime} = T(s, a)$;
and for function $g$ over $\mathcal{X}$, $\|g\|_{p,\mu}$ is defined as $\|g\|_{p, \mu}^{p}=\int|g(x)|^{p} \mu(d x)$. Following \citet{munos2008finite}, we begin with the error bound in a single iteration.

\newtheorem{lemma}{Lemma}[section]
\begin{lemma}
\label{Lemma:Single Iteration Error Bound}
(Single Iteration Error Bound) Let $V_k$ and $V_{k+1}$ be the value functions of iteration $k$ and $k+1$, and $V_{max} = r_{max}/(1-\gamma)$. For $p \geq 1$ and a certain state distribution $\mu$, let the inherent Bellman error of the value function space $\mathcal{F}$ be defined by $d_{p, \mu}(B \mathcal{F}, \mathcal{F})=\sup_{V \in \mathcal{F}} \inf _{f \in \mathcal{F}}\|f-B V\|_{p, \mu}$. Assume the value functions are $K_{V}$-$Lipschitz$ continuous, i.e., for any k and states pair $(s_i, s_j)$, it holds that $|V_k(s_i)-V_k(s_j)| \leq K_{V} \|s_i-s_j \|_2$. And assume the $L^2$-norm model error between the real next state $s^{\prime}$ and the predicted ones $\hat{s^{\prime}}$ obeys a half-normal distribution, i.e., the probability density function $f(\|\hat{s^{\prime}}-s^{\prime}\|_2)=\frac{2}{\sqrt{2 \pi} \sigma} \exp \left(-\frac{\|\hat{s^{\prime}}-s^{\prime}\|_2^{2}}{2 \sigma^{2}}\right)$. Let $\mathcal{N}_{0}(N)=\mathcal{N}\left(\frac{1}{8}\left(\frac{\epsilon_0}{4}\right)^{p}, \mathcal{F}, N, \mu\right)$ denotes the covering number defined in Lemma \ref{lemma:Pollard}. Then for any $\epsilon_0,\delta_0 > 0$, the inequality
\begin{equation}
    \nonumber
    \left\|V_{k+1}-B V_k\right\|_{p, \mu} \leq d_{p, \mu}(B \mathcal{F}, \mathcal{F})+\epsilon_0
\end{equation}
holds w.p. at least $1-\delta_0$ provided that
\begin{equation}
    N>128\left(\frac{8 V_{\max }}{\epsilon_0}\right)^{2 p}\Big(\log (1 / \delta_0)+\log \left(32 \mathcal{N}_{0}(N) \right)\Big)
\end{equation}
and
\begin{equation}
    \sigma<\frac{\epsilon_0}{4\gamma K_V \Phi^{-1}\Big(1-\delta_0/\big(8N|A|(1-\beta)\big)\Big)}.
\end{equation}

\end{lemma}
\begin{proof}
See \emph{Appendix \ref{appendix:proofs}}, \emph{Lemma \ref{appendix Lemma:Single Iteration Error Bound}}. 
\end{proof}
According to the model error assumption (more details in Appendix \ref{appendix:model error assumption}), the model will give precise predictions if $\sigma$ is small enough. Therefore, Lemma \ref{Lemma:Single Iteration Error Bound} states that with an accurate model and sufficient large $N$, the errors between $B V_k$ and $V_{k+1}$ in each iteration can be bounded with high probability. Now let us turn to bound the return discrepancy between the greedy policy w.r.t. $V_K$ and the optimal policy.
\begin{theorem}
\label{theorem: return bound}
($\beta$-mixture sampling-based FVI bound) Assume that the pseudo-dimension of the value function space is finite as in Proposition \ref{proposition:pseudo-dimension}, under the concentrability assumption and the same assumptions of Lemma \ref{Lemma:Single Iteration Error Bound}, let $\rho$ be an arbitrary state distribution, $C_{\rho,\mu}$ be the discounted-average concentrability coefficient and $\pi_{K}$ be the greedy policy w.r.t. $V_K$. Let $V^{\pi_{K}}$ and $V^{*}$ be the expected return of executing $\pi_{K}$ and the optimal policy $\pi^{*}$ in real environment, respectively. Define $N_{\rm{real}} = N \cdot |\mathcal{A}|\cdot \beta$ as the expected number of real samples. Then the following bound holds w.p. at least $1-\delta$:
\begin{equation}
\begin{aligned}
    \|V^{*}-V^{\pi_{K}}\|_{p, \rho} \leq & \frac{2 \gamma}{(1-\gamma)^{2}} C_{\rho, \mu}^{1 / p} d_{p, \mu}(B \mathcal{F}, \mathcal{F})
+ O\left(\Big(\frac{\beta |\mathcal{A}|}{N_{\rm{real}}}\big(\log (\frac{N_{\rm{real}}}{\beta |\mathcal{A}|})+\log (\frac{K}{\delta})\big)\Big)^{\frac{1}{2p}}\right) 
\\&+O\Big(\Phi^{-1}\big(1-\frac{\beta\delta}{8KN_{\rm{real}}(1-\beta)}\big) \sigma\Big)
+O\left(\gamma^{K/p} V_{\max }\right). 
\end{aligned}    
\end{equation}

\end{theorem}
\begin{proof}
See \emph{Appendix \ref{appendix:proofs}}, \emph{Theorem \ref{appendix theorem: return bound}}. 
\end{proof}


\textbf{Remark.} The definition of the concentrability assumption is provided in Appendix \ref{appendix:proofs}. Theorem \ref{theorem: return bound} gives an upper bound on the return discrepancy. To achieve small return discrepancy, we hope to minimize the upper bound.
For the first term in the bound, the approximation error $d_{p, \mu}(B \mathcal{F}, \mathcal{F})$ can be made small by selecting $\mathcal{F}$ to be large enough, and the last term arises due to the finite number of iterations. For the remaining two terms, as $\beta / N_{\rm{real}}$ increases, the second term increases while the third term decreases. Therefore, there exists an optimal value for $\beta / N_{\rm{real}}$ to trade off these two terms. In online MBRL, $N_{\rm{real}}$ increases as real samples are continuously collected from the environment throughout the training process, so that gradually increasing the real ratio $\beta$ is promising to achieve good performance according to the upper bound.

From another perspective, the second term corresponds to the estimation error of using limited data to learn the value functions. By utilizing a learned model, we can generate as much data as we want, which, however, tends to suffer from the model error represented by the third term. The hyperparameter $\beta$ serves as a trade-off between these two kinds of errors. Intuitively, with more real data available, the pressure of limited data is alleviated, which means we can increase $\beta$ accordingly to reduce the reliance on inaccurate models.

\section{The AutoMBPO Framework}
The analysis in the previous section shows that the optimal real ratio $\beta$ during training should gradually increase according to the current optimization states, such as the number of real samples and model error. These states may also be affected by other hyperparameters, e.g., model training frequency and rollout length, which makes the hyperparameter scheduling non-trivial in practice. In order to investigate the optimal joint scheduling of these hyperparameters, we propose a practical framework named AutoMBPO to schedule these hyperparameters automatically when training an MBRL algorithm. The high-level idea of AutoMBPO is to formulate hyperparameter scheduling as a sequential decision-making problem and then adopt some RL algorithm to solve it. Here we choose MBPO \cite{mbpo} as a representative running case since MBPO has achieved remarkable success in continuous control benchmark tasks and can naturally take all the hyperparameters we care about into consideration (a comparison of the hyperparameters that can be scheduled in different MBRL algorithms is provided in Appendix \ref{appendix:algo comparison}). We illustrate the overview framework of AutoMBPO in Figure \ref{fig:framework} and detail the algorithm in Algorithm~\ref{alg}.

\subsection{Hyper-MDP Formulation}
\label{sec:Hyper-MDP Formulation}
We formulate the problem of determining the hyperparameter schedule as a new sequential decision-making problem, which we will call \emph{hyper-MDP}. In the hyper-MDP, the hyper-controller needs to decide the hyperparameter values throughout one complete run of the MBPO learning process. Below we elaborate on several critical aspects of the hyper-MDP formulation used in this paper.

\textbf{Episode.} An episode of the hyper-MDP consists of the whole process of training an MBPO instance from scratch, which contains multiple episodes of the target-MDP. We will terminate an episode of the hyper-MDP if $N$ samples have been collected in the target-MDP. To be more specific, if the episode length of the target-MDP is $H$, and we only train MBPO for $m$ episodes, then we have $N=H\cdot m$.

\begin{wrapfigure}{r}{0.61\textwidth}
    \vspace{-7pt}
    \begin{center}
    \includegraphics[width = 0.61\textwidth]{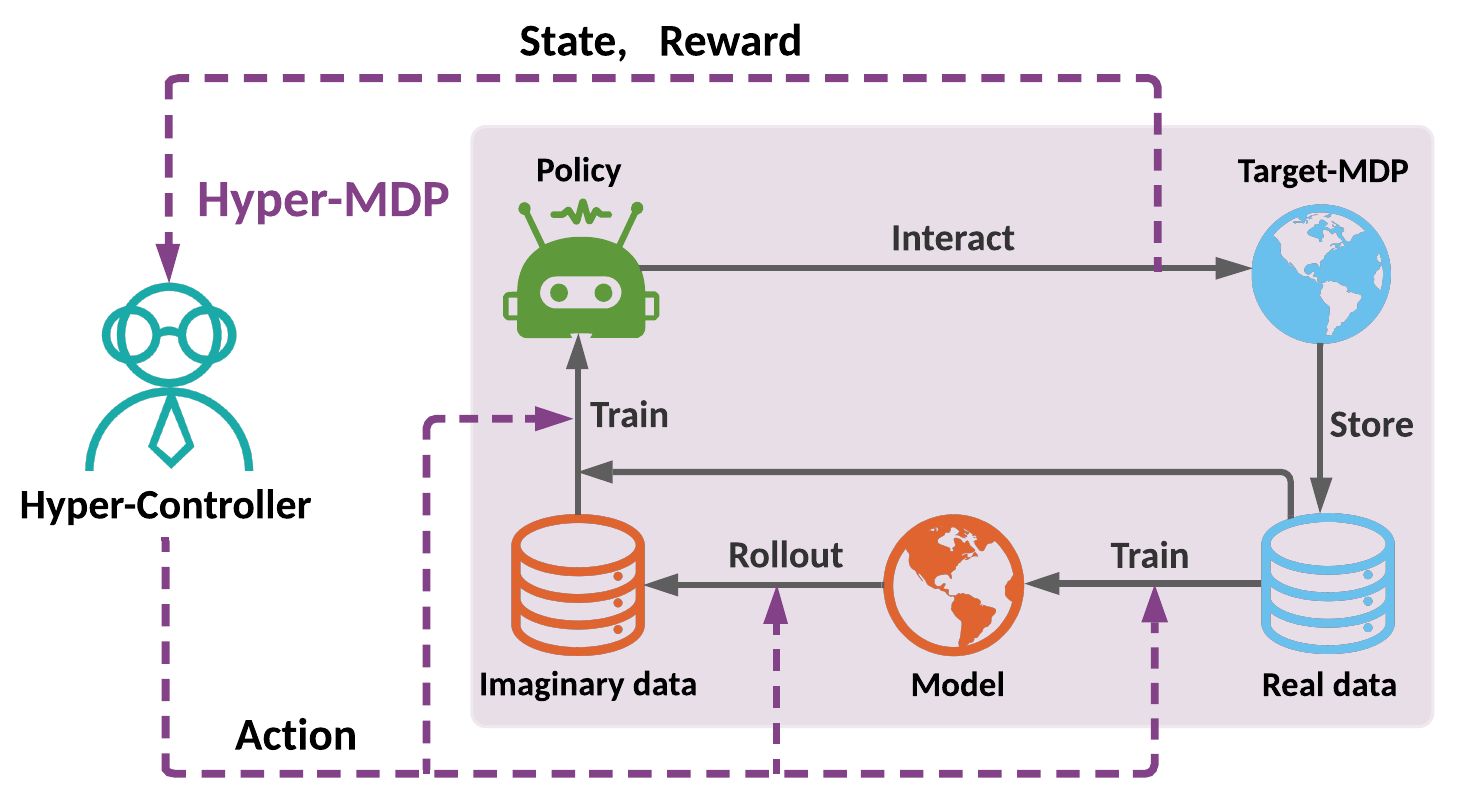}
    \end{center}
    \vspace{-7pt}
    \caption{Illustration of AutoMBPO framework. The hyper-controller (left) treats the whole learning process of MBPO (right) as its hyper-MDP and takes action according to the state containing all the information of MBPO training. The action is to adjust the hyperparameters of MBPO, e.g., real ratio, model training frequency, policy training iteration, and rollout length. Afterward, a reward is given to the hyper-controller based on the performance of MBPO policy in target-MDP.}
    \vspace{-10pt}
    \label{fig:framework}
\end{wrapfigure}

\textbf{State.} The state should be defined to include all the information of MBPO training to help hyper-controller make decisions accordingly. We use a concatenation of several features to represent the state. As suggested by Theorem \ref{theorem: return bound}, the state should contain $N_{\rm{real}}$, the number of real samples so far, and the model error, which can be estimated by the model loss $\mathcal{L}_{\hat{T}}$ on a validation set. Besides, we also add other features to better capture the current optimization situation: (1) critic loss: $\mathcal{L}_{Q}$, which to some extent represents how well the value function fits the training data; (2) policy change: $\epsilon_{\pi} = E_{(s,a) \sim \pi_{D}}[|\pi(a|s)- \pi_{D}(a|s)|]$, which evaluates the distance of current policy $\pi$ and the data-collecting policy $\pi_{D}$ on sampled data. In practice, we calculate $\epsilon_{\pi}$ using recently collected data to reflect the updating speed of policy roughly;
(3) validation metrics: we use the average return on the target-MDP to represent the performance of the policy;
(4) current hyperparameters: the hyperparameter values used currently by MBPO. All these features are normalized into $[0,1]$ to make the learning process stable.

\textbf{Action.} 
The action set in the hyper-MDP should contain all the hyperparameters we care about. Firstly, we would like to adjust the ratio of real data when optimizing the policy. Secondly, we also consider investigating the model training and policy learning frequency to balance their alternate optimization \cite{rajeswaran2020game}. Lastly, the rollout length when generating imaginary rollouts is crucial in MBPO, but it remains unclear how to choose its value \cite{mbpo}. Besides, since there exists an optimal value for $\beta / N_\text{real}$ according to the theoretical analysis and $N_\text{real}$ gradually increases during training in practice, for the action space, adjusting the real ratio by multiplying a constant is enough to get an effective schedule.
To sum up, the action set and the corresponding space of each action are: (1) adjusting the real ratio $\beta$: $\{\times c^{-1}, \times 1 , \times c\}$, where c is a constant larger than $1$; (2) deciding whether to train the model: $\{0,1\}$. Note that in MBPO, the model training iterations are implicitly decided by an early stopping trick, so we only decide whether to train the model; (3) adjusting the policy training iteration $G$: $\{-1, 0, +1\}$; (4) adjusting the model rollout length $k$: $\{-1, 0, +1\}$. In practice, the actions are taken every $\tau$ real samples are collected where $\tau$ is the hyperparameter of the hyper-controller.

\textbf{Reward.} One direct way is to define a reward at the end of an episode in hyper-MDP, which, however, will lead to the sparse reward problem. We define our reward as follows: every $H$ real samples, i.e., one episode in target-MDP, we evaluate the MBPO policy on target-MDP and define the reward as the average evaluated return.

\begin{algorithm2e*}[t]
\label{alg}
\caption{AutoMBPO}
\LinesNumbered
\DontPrintSemicolon
\label{alg:AutoMBPO}
Initialize a hyper-controller $\textbf{MC}$\;
\Repeat{acceptable performance is achieved}{
      Initialize a MBPO instance; initialize the hyper-controller sample count $i = 0$\;
      \For{$m$ target-MDP episodes}{
\For{$h$ = $0:H-1$ real timesteps (an episode in target-MDP)}{
Interact with the target-MDP using the MBPO policy; add the real samples to $\mathcal{D}_{\mathrm{env}}$\;
\If{$h$ \text{mod} $\tau=0$}{
Calculate the current state $\bold{S}_i$ for hyper-MDP; take an action based on $\bold{S}_i$: $\bold{A}_i \gets \textbf{MC}(\bold{S}_i)$\;
Decide whether to train the model, and adjust the hyperparameters $k,G,\beta$ according to $\bold{A}_i$\;
}
\For{F model rollouts}{
Sample $s_{t}$ uniformly from $\mathcal{D}_{\mathrm{env}}$; perform $k$-step rollout from $s_t$; add the imaginary data to $\mathcal{D}_{\mathrm{model}}$\;
}
\For{$G$ gradient updates}{
Train the MBPO policy using SAC on $\mathcal{D}_{\mathrm{env}}$ and $\mathcal{D}_{\mathrm{model}}$ in the ratio $\beta$\;
}
\uIf{$h = H-1$ }{
Evaluate MBPO and construct the reward as the average return: $\bold{R}_{i} = Avg(\eta)$; update $i = i+1$\;
}
\ElseIf{$h$ \text{mod} $\tau=0$ and $h \not= H-\tau$}{Construct the reward $\bold{R}_{i}=0$; update $i = i+1$}
}
}
Train the hyper-controller $\textbf{MC}$ on the data $\{(\bold{S}_i,\bold{A}_i,\bold{R}_i)\}_{i = 0}^{m \cdot H/\tau}$
    }
\end{algorithm2e*}

\subsection{Hyper-Controller Learning}
After formulating the hyper-MDP, we are now ready to learn a hyper-controller in the hyper-MDP. In fact, we can use any off-the-shelf reinforcement learning algorithm, and we choose Proximal Policy Optimization (PPO) \cite{schulman2017proximal} as the hyper-controller algorithm since PPO has achieved considerable success in the optimization schedule problem \cite{autoloss}. To be more specific, we maximize the following clipped PPO objective:
\begin{equation}
\nonumber
    \mathcal{J}(\Omega)=\hat{\mathbb{E}}_{t}\left[\min  \big(r_{t}(\Omega) \hat{A}_{t}, \operatorname{clip}\left(r_{t}(\Omega), 1-\epsilon, 1+\epsilon\right) \hat{A}_{t}\big)\right],
\end{equation}
where $r_t(\Omega) = \textbf{MC}_{\Omega}(a_{t} | s_{t})/\textbf{MC}_{\Omega_{\text{old}}}(a_{t} | s_{t})$ and $\textbf{MC}_\Omega$ is the hyper-controller policy parameterized by $\Omega$.

There are multiple choices of the advantage function \cite{gae}, and we use the baseline version of the Monte-Carlo returns to reduce the variance:
\begin{equation}
    \hat{A}_{t} = \Sigma_{i = t}^{m \cdot H / \tau} (R_i - R^{\prime}_{i}).
\end{equation}
Here, $R^{\prime}_{i}$ is the average return of MBPO policy trained in advance with the same amount of real data using the original hyperparameters, and we treat $\Sigma_{i = t}^{m \cdot H / \tau} R^{\prime}_{i}$ as the baseline. The advantage function can also be regarded as the performance improvement compared with the original MBPO configuration.

\section{Experiment}
\label{sec: exp}
In this section, we conduct experiments to verify the effectiveness of AutoMBPO and provide comprehensive empirical analysis. Our experiments aim to answer the following three questions: i) How does the MBPO instance learned by AutoMBPO perform compared to the original configuration?
ii) Do the learned hyper-controllers in different target-MDPs share similar schedules, and are the schedules consistent with previous theoretical analysis? 
iii) Which hyperparameter in MBPO is more important/sensitive to the final performance? More experimental results are provided in Appendix \ref{appendix:more experiments results}.

\begin{figure*}[htb]
	\centering
	\includegraphics[width=1\textwidth]{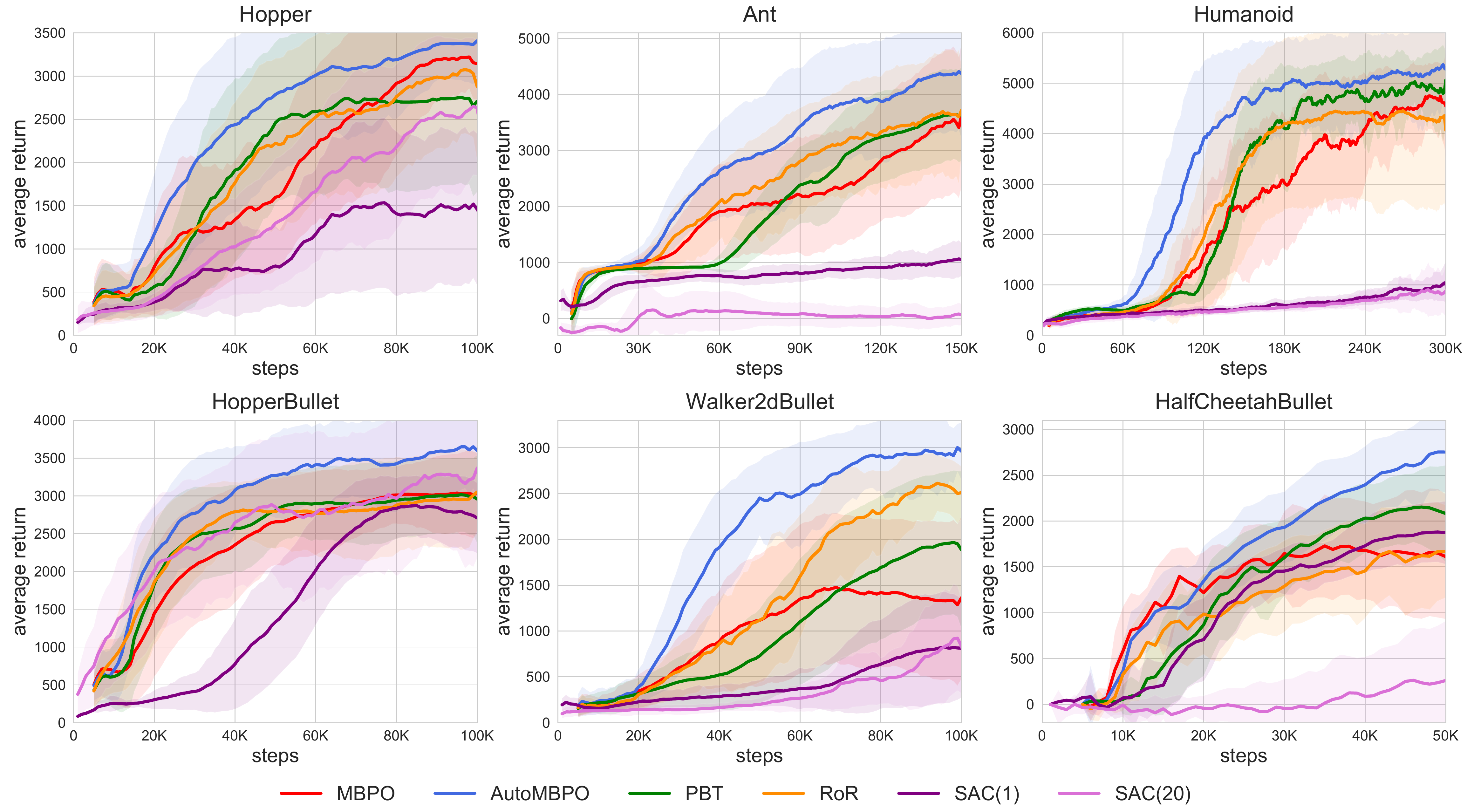}
	\vspace{-10pt}
	\caption{Learning curves of AutoMBPO and baselines on different continuous control environments. The solid lines indicate the mean, and shaded areas indicate the standard deviation of 10 trials over different random seeds. Each trial is evaluated every 1000 environment steps, where each evaluation reports the average return over 10 episodes.
	}
	
	\vspace{-5pt}
	\label{fig:comparison result}
\end{figure*}

\subsection{Comparative Evaluation}
\label{sec:compare}
We mainly compare the MBPO instance found by \textbf{AutoMBPO} to the \textbf{MBPO} with original hyperparameters \cite{mbpo} since AutoMBPO is directly based on MBPO, and there is no need to compare to other model-based methods. 
Besides, we also compare to Soft Actor-Critic (SAC) \cite{sac}, the state-of-the-art model-free algorithm in terms of sample efficiency and asymptotic performance. Actually, SAC can be viewed as an extreme variant of MBPO where the real ratio in training policy is 1. To make a fair comparison, we also provide the result of SAC where the policy training iteration is enforced to be the same as in MBPO. To be more specific, the original SAC, denoted as \textbf{SAC(1)}, trains the policy one time per step in the real environment, while in MBPO, the policy is trained 20 times per real step, so we add the baseline \textbf{SAC(20)}. We also compare our method to the \textbf{PBT} \cite{zhang2021importance} and \textbf{RoR} \cite{dong2020intelligent} algorithm for a comprehensive comparison. Note that the hyper-controllers in AutoMBPO and RoR are trained in advance and are only evaluated here. The whole training process of AutoMBPO is provided in Appendix \ref{appendix:effectiveness}.

The compared methods are evaluated on three MuJoCo \cite{todorov2012mujoco} (\textbf{Hopper, Ant, Humanoid}), and three PyBullet
\cite{coumans2016pybullet} (\textbf{HopperBullet, Walker2dBullet, HalcheetahBullet}) continuous control tasks as our target-MDPs\footnote{Code is available at: \href{https://github.com/hanglai/autombpo}{https://github.com/hanglai/autombpo}}. In practice, the three hyperparameters (real ratio, model training frequency, and policy training iteration) are scheduled on all tasks, while the rollout length is not scheduled on the PyBullet tasks since preliminary experiments showed that even 2-step rollouts degrade the performance on the PyBullet environments. Moreover, due to the high computational cost, the hyper-controller is only trained in previous $m$ episodes of MBPO and will be used to continuously control MBPO instances running for $M$ episodes, which is usually $2-3$ times larger than $m$ (e.g., $m = 50, M = 150$ on Ant). Furthermore, we give a penalty for the model training to speed up the training process, and we find it does not influence the results much. More experimental details and hyperparameter configuration for hyper-controller can be found in Appendix \ref{appendix:experimental settings} and \ref{appendix:hyperparameters}.

The comparison results are shown in Figure \ref{fig:comparison result}. We can observe that: i) the MBPO instances learned by AutoMBPO significantly outperform MBPO with original configuration in all tasks, highlighting the strength of automatic scheduling. ii) Compared to other hyperparameter optimization baselines (i.e., PBT and RoR), our approach AutoMBPO achieves better performance and generalization in different environments, verifying the effectiveness of our proposed framework. iii) MFRL methods with exhaustive policy training may overfit to a small amount of real data since SAC(20) achieves worse performance than SAC(1) on some complex tasks like Ant and Walker2dBullet. The overfitting problem is much less severe in MBPO since we can generate plentiful data using the model. iv) Though only trained in previous episodes, the hyper-controller can still find reasonable hyperparameter schedules and can be extended to the longer training process of MBPO.

\subsection{Hyperparameter Schedule Visualization}
\label{sec:hyperparameter visualization}
We visualize the schedules of the hyperparameters found by AutoMBPO in Figure \ref{fig:hyperparameter}. From the results, the hyperparameter schedules on different tasks show both similarities and differences: i) A similar increasing schedule of the real ratio can be observed in all six environments, which is consistent with our derivation in Section \ref{sec:theory} and the empirical results in \citet{kalweit2017uncertainty}. ii) The hyper-controller tends to train the model more frequently in complex environments, such as Ant and Humanoid.
iii) To a certain extent, increasing the policy training iteration can better exploit the model. However, it may also increase the risk of instability \cite{rajeswaran2020game}, and we observe that the hyper-controller decreases the policy training iteration in the subsequent training phase of Hopper when the policy is close to convergence. iv) As for the rollout length, the hyper-controller adopts a nearly linear increasing schedule, which is close to the manually designed one used in the original MBPO \cite{mbpo}.

\begin{figure*}[htb]
	\centering
 	\vspace{-5pt}
	\includegraphics[width=1\textwidth]{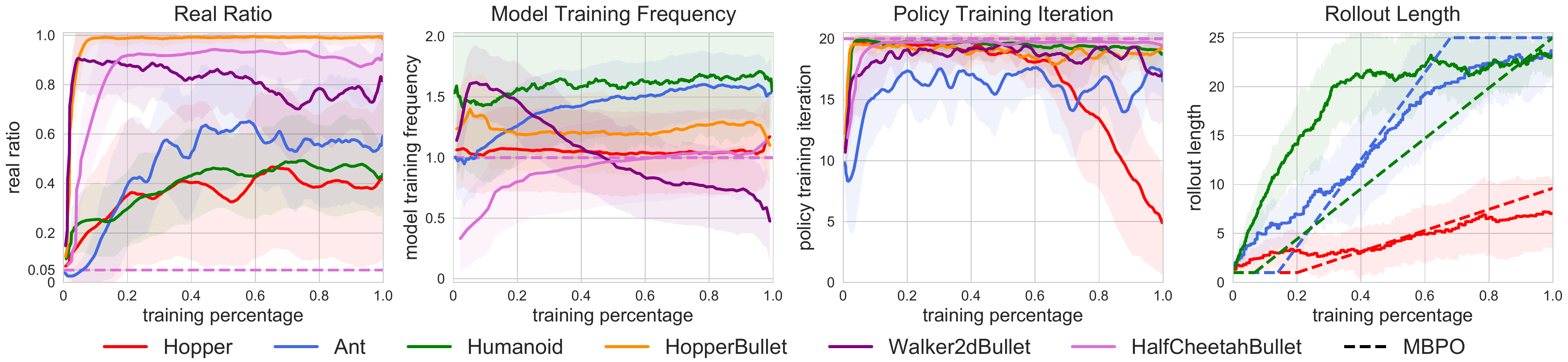}
	\vspace{-10pt}
	\caption{Hyperparameter schedule found by AutoMBPO on different tasks. The dashed lines indicate the original hyperparameter configuration in MBPO. For better visualization, model training frequency is scaled by the frequency of the original MBPO, and the training process (x-axis) is scaled to [0,1] since the max episodes vary on different tasks.}
	
 	\vspace{-5pt}
	\label{fig:hyperparameter}
\end{figure*}

\subsection{Importance of Hyperparameters}
\label{sec:Importance of Hyperparameters}

In this section, we investigate the importance of different hyperparameters. Specifically, we utilize the hyper-controller trained previously to schedule one certain hyperparameter respectively, i.e., real ratio (AutoMBPO-R), model training frequency (AutoMBPO-M), policy training iteration (AutoMBPO-P), and rollout length (AutoMBPO-L), while fixing other hyperparameters to the same as in the original MBPO. As shown in Figure \ref{fig:importance}, the importance of various hyperparameters changes across different tasks: using the hyper-controller to schedule the policy training iteration is helpful on Hopper but fails on Ant, while the situation reverses for model training frequency. However, using the hyper-controller to schedule the real ratio retains much of the benefit of AutoMBPO on almost all tasks, which further verifies the importance of real ratio.

\begin{figure*}[htb]
	\centering
	\vspace{-5pt}
	\includegraphics[width=1\textwidth]{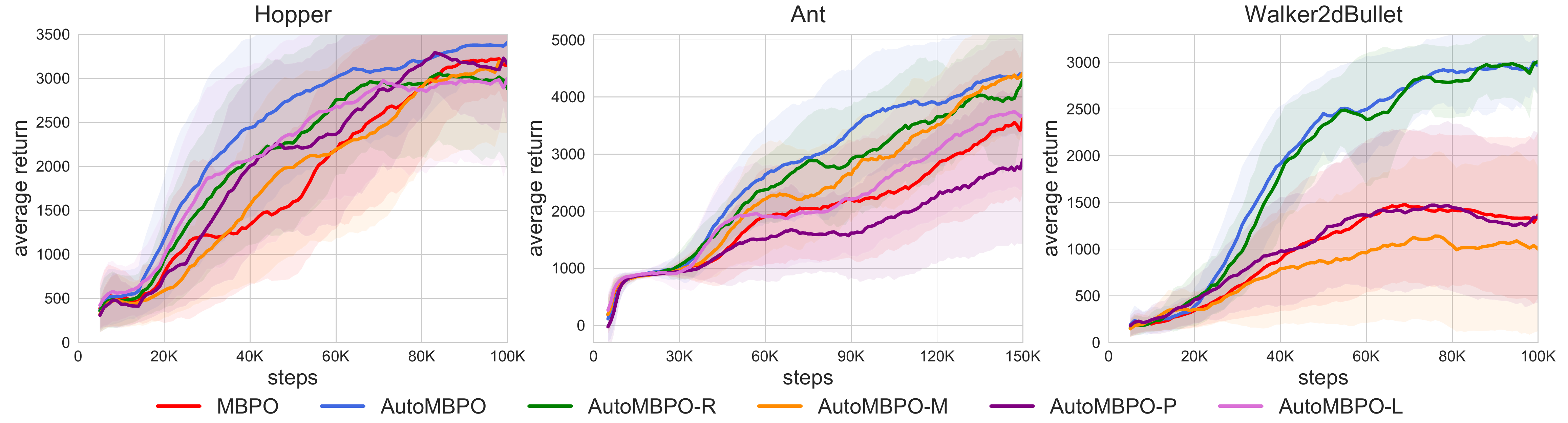}
	\vspace{-10pt}
	\caption{Performance of the instances with only one specific hyperparameter scheduled by hyper-controller while other hyperparameters fixed to the same as in MBPO. The complete figures of all six environments are provided in Appendix \ref{sec:full figure}.}
	
 	\vspace{-5pt}
	\label{fig:importance}
\end{figure*}

\section{Conclusion}
In this paper, we conduct a theoretical analysis to justify real data usage for policy optimization in MBRL algorithms. Inspired by the analysis, we present a framework called AutoMBPO to automatically schedule the hyperparameters in the MBPO algorithm. Extensive experiments on complex continuous control tasks demonstrate that the derived MBPO instance with hyperparameters scheduled by AutoMBPO can achieve significantly better performance than the MBPO instance of the original configuration. Moreover, the real ratio schedule discovered by AutoMBPO is consistent with our theoretical analysis. We hope this insight can help design more effective MBRL algorithms. For future work, we plan to apply our hyper-controller training scheme to other MBRL methods.

\section*{Acknowledgments}
The SJTU team is supported by ``New Generation of AI 2030'' Major Project (2018AAA0100900), Shanghai Municipal Science and Technology Major Project (2021SHZDZX0102) and National Natural Science Foundation of China (62076161, 61772333, 61632017). The work is also sponsored by Huawei Innovation Research Program.

\bibliographystyle{plainnat}
\bibliography{reference}

\newpage
\appendix

\newcommand{\tabincell}[2]{\begin{tabular}{@{}#1@{}}#2\end{tabular}}

\section*{\centering Appendix for: On Effective Scheduling of \\ Model-based Reinforcement Learning}
\section{Proofs}
\label{appendix:proofs}
\newtheorem{assumption}{Assumption}[section]
\begin{assumption}
\label{assumption:concentrability}
(Discounted-average concentrability of future-state distributions) Given $\rho,\mu,m \geq 1$, and an arbitrary sequence of stationary policies $\{\pi_m\}_{m \geq 1}$, assume that the future-state distribution $\rho T^{\pi_{1}} T^{\pi_{2}} \ldots T^{\pi_{m}}$ is absolutely continuous w.r.t. $\mu$. Assume that
\begin{equation}
    \nonumber
    c(m) \stackrel{\rm def}{=} \sup _{\pi_{1}, \ldots, \pi_{m}}\left\|\frac{d\left(\rho T^{\pi_{1}} T^{\pi_{2}} \ldots T^{\pi_{m}}\right)}{d \mu}\right\|_{\infty}
\end{equation}
satisfies 
\begin{equation}
    \nonumber
    C_{\rho, \mu} \stackrel{\rm def}{=}(1-\gamma)^{2} \sum_{m \geq 1} m \gamma^{m-1} c(m)<+\infty.
\end{equation}
\end{assumption}
We call $c(m)$ the \emph{$m$-step concentrability} of a future-state distribution and call $C_{\rho,\mu}$ the \emph{discounted-average concentrability coefficient} of the future-state distributions. The class of MDPs that satisfies this concentrability assumption is quite large, which is further discussed in \citet{munos2008finite}.

\begin{lemma}
\label{lemma:Pollard}
\textbf{(Pollard, 1984)} Let $\mathcal{F}$ be a set of measurable functions $f:\mathcal{X} \rightarrow [0,K]$ and let $\epsilon > 0$, $N$ be arbitrary. If $X_i$, $i = 1,\ldots,N$ is an i.i.d. sequence taking values in the space $\mathcal{X}$ then 
\begin{equation}
    \nonumber
    \mathbb{P}\left(\sup _{f \in \mathcal{F}}\left|\frac{1}{N} \sum_{i=1}^{N} f\left(X_{i}\right)-\mathbb{E}\left[f\left(X_{1}\right)\right]\right|>\epsilon\right) \leq 8 \mathbb{E}\left[\mathcal{N}\left(\epsilon / 8, \mathcal{F}\left(X^{1: N}\right)\right)\right] e^{-\frac{N \epsilon^{2}}{128 K^{2}}},
\end{equation}
where $\mathcal{N}_q \left(\epsilon, \mathcal{F}\left(X^{1: N}\right)\right)$ denotes the ($\epsilon$,q)-covering number of the set $\mathcal{F}(X^{1:N}) = \left\{\left(f\left(X_{1}\right), \ldots, f\left(X_{N}\right)\right) \mid f \in \mathcal{F}\right\}$, i.e., the smallest integer $m$ such that $\mathcal{F}(X^{1:N})$ can be covered by $m$ balls of the normed-space $\left(\mathbb{R}^{N},\|\cdot\|_{q}\right)$ with centers in $\mathcal{F}(X^{1:N})$ and radius $N^{1/q}\epsilon$. When $X^{1:N}$ are i.i.d with the distribution $\mu$, we denote $\mathbb{E}\left[\mathcal{N}_{q}\left(\epsilon, \mathcal{F}\left(X^{1: N}\right)\right)\right]$ as $\mathcal{N}_{q}(\epsilon, \mathcal{F}, N, \mu)$. 
And when $q=1$, $\mathcal{N}$ is used instead of $\mathcal{N}_{1}$.
\end{lemma}

\newtheorem{proposition}{Proposition}[section]
\begin{proposition}
\label{proposition:pseudo-dimension}
\textbf{(Haussler $(1995)$, Corollary 3$)$} For any set $X$, any points $x^{1: N} \in X^{N}$, any class $\mathcal{F}$ of functions on $\mathcal{X}$ taking values in $[0, L]$ with pseudo-dimension $V_{\mathcal{F}^{+}}<\infty$, and any $\epsilon>0$,
$$
\mathcal{N}\left(\epsilon, \mathcal{F}\left(x^{1: N}\right)\right) \leq e\left(V_{\mathcal{F}+}+1\right)\left(\frac{2 e L}{\epsilon}\right)^{V_{\mathcal{F}+}} .
$$
\end{proposition}

\begin{lemma}
\label{appendix Lemma:Single Iteration Error Bound}
(Single Iteration Error Bound) Let $V_k$ and $V_{k+1}$ be the value functions of iteration $k$ and $k+1$, and $V_{max} = r_{max}/(1-\gamma)$. For $p \geq 1$ and a certain state distribution $\mu$, let the inherent Bellman error of the value function space $\mathcal{F}$ be defined by $d_{p, \mu}(B \mathcal{F}, \mathcal{F})=\sup_{V \in \mathcal{F}} \inf _{f \in \mathcal{F}}\|f-B V\|_{p, \mu}$. Assume the value functions are $K_{V}$-$Lipschitz$ continuous, i.e., for any k and states pair $(s_i, s_j)$, it holds that $|V_k(s_i)-V_k(s_j)| \leq K_{V} \|s_i-s_j \|_2$. And assume the $L^2$-norm model error between the real next state $s^{\prime}$ and the predicted ones $\hat{s^{\prime}}$ obeys a half-normal distribution, i.e., the probability density function $f(\|\hat{s^{\prime}}-s^{\prime}\|_2)=\frac{2}{\sqrt{2 \pi} \sigma} \exp \left(-\frac{\|\hat{s^{\prime}}-s^{\prime}\|_2^{2}}{2 \sigma^{2}}\right)$. Let $\mathcal{N}_{0}(N)=\mathcal{N}\left(\frac{1}{8}\left(\frac{\epsilon_0}{4}\right)^{p}, \mathcal{F}, N, \mu\right)$ denotes the covering number defined in Lemma \ref{lemma:Pollard}. Then for any $\epsilon_0,\delta_0 > 0$, the inequality
\begin{equation}
    \nonumber
    \left\|V_{k+1}-B V_k\right\|_{p, \mu} \leq d_{p, \mu}(B \mathcal{F}, \mathcal{F})+\epsilon_0
\end{equation}
holds w.p. at least $1-\delta_0$ provided that
\begin{equation}
\label{equa:bound for N}
    N>128\left(\frac{8 V_{\max }}{\epsilon_0}\right)^{2 p}\Big(\log (1 / \delta_0)+\log \left(32 \mathcal{N}_{0}(N) \right)\Big)
\end{equation}
and
\begin{equation}
\label{equa:bound for sigma}
    \sigma<\frac{\epsilon_0}{4\gamma K_V \Phi^{-1}\Big(1-\delta_0/\big(8N|A|(1-\beta)\big)\Big)}.
\end{equation}
\end{lemma}

\begin{proof}
Let $f^*$ be the best fit to $BV_{k}$ in $\mathcal{F}$: $f^{*} = \Pi_{\mathcal{F}} B V_k$, and let $\hat{\mu}$ denotes the distribution of the $N$ sampled states $\{s_i\}_{i=1}^{N}$. Define $\|\cdot\|_{p,\hat{\mu}}$ as:
\begin{equation}
    \nonumber
    \|f\|^{p}_{p,\hat{\mu}} = \frac{1}{N} \sum_{i=1}^{N}\left|f\left(s_{i}\right)\right|^{p}.
\end{equation}

If the following inequalities hold simultaneously w.p. not smaller than $1-\delta_0$, then the Lemma can be proved by choosing $\epsilon_0^{\prime} = \epsilon_0/4$:
\begin{align}
\left\|V_{k+1}-B V_k\right\|_{p, \mu} & \leq\left\|V_{k+1}-B V_k\right\|_{p, \hat{\mu}}+\epsilon_0^{\prime} \label{single iter inequality 1}
\\& \leq\|V_{k+1}-\hat{V}_{k}\|_{p, \hat{\mu}}+2 \epsilon_0^{\prime} \label{single iter inequality 2}
\\& \leq\|f^{*}-\hat{V}_{k}\|_{p, \hat{\mu}}+2 \epsilon_0^{\prime} \label{single iter inequality 3}
\\& \leq\left\|f^{*}-B V_k\right\|_{p, \hat{\mu}}+3 \epsilon_0^{\prime} \label{single iter inequality 4}
\\& \leq\left\|f^{*}-B V_k\right\|_{p, \mu}+4 \epsilon_0^{\prime} \label{single iter inequality 5}
\\
\nonumber
& \leq d_{p, \mu}(B \mathcal{F}, \mathcal{F})+4 \epsilon_0^{\prime} 
\end{align}

Since $V_{k+1}$ is the best fit to $\hat{V}_{k}$ in $\mathcal{F}$, $(\ref{single iter inequality 3})$ holds w.p. 1. So we only need to prove that $(\ref{single iter inequality 1})$, $(\ref{single iter inequality 2})$, $(\ref{single iter inequality 4})$, $(\ref{single iter inequality 5})$ hold w.p. at least $1-\delta_0^{\prime}$ where $\delta_0^{\prime} = \delta_0/4$. The proof of $(\ref{single iter inequality 1})$ and $(\ref{single iter inequality 5})$ is the same as that in (\citet{munos2008finite}, Lemma1), and we provide a quick proof here for completeness. Let
\begin{equation}
    \nonumber
    Q=\max \left(\left|\left\|V_{k+1}-B V_k\right\|_{p, \mu}-\left\|V_{k+1}-B V_k\right\|_{p, \hat{\mu}}\right|,\left|\left\|f^{*}-B V_k\right\|_{p, \mu}-\left\|f^{*}-B V_k\right\|_{p, \hat{\mu}}\right|\right).
\end{equation}
So $(\ref{single iter inequality 1})$ and $(\ref{single iter inequality 5})$ will follow if
\begin{equation}
\label{inequa:Q_epsilon}
    \mathbb{P}\left(Q>\epsilon_0^{\prime}\right) \leq \delta_0^{\prime}.
\end{equation}
And due to the inequality:
\begin{equation}
    Q \leq \sup _{f \in \mathcal{F}}\Big|\|f-B V_k\|_{p, \mu}-\|f-B V_k\|_{p, \hat{\mu}}\Big|,
\end{equation}
we have:
\begin{equation}
\label{inequal_to_combine_1}
    \mathbb{P}\left(Q>\epsilon_0^{\prime}\right) \leq \mathbb{P}\left(\sup _{f \in \mathcal{F}}\Big|\|f-B V_k\|_{p, \mu}-\|f-B V_k\|_{p, \hat{\mu}}\Big|>\epsilon_0^{\prime}\right).
\end{equation}
For any event $\omega$ such that
\begin{equation}
    \sup _{f \in \mathcal{F}}\Big|\|f-B V_k\|_{p, \mu}-\|f-B V_k\|_{p, \hat{\mu}}\Big|>\epsilon_0^{\prime},
\end{equation}
there exist a function $f^{\prime} \in \mathcal{F}$ such that
\begin{equation}
    \Big|\|f^{\prime}-B V_k\|_{p, \mu}-\|f^{\prime}-B V_k\|_{p, \hat{\mu}}\Big|>\epsilon_0^{\prime}.
\end{equation}
First assume that $\|f^{\prime}-B V_k\|_{p, \hat{\mu}} \leq \|f^{\prime}-B V_k\|_{p, \mu}$. Hence, $\|f^{\prime}-B V_k\|_{p, \hat{\mu}} + \epsilon_0^{\prime} \leq \|f^{\prime}-B V_k\|_{p, \mu}$. Since the elementary inequality $x^p + y^p \leq (x+y)^{p}$ holds for $p \geq 1$ and any non-negative x, y, we can get $\|f^{\prime}-B V_k\|_{p, \hat{\mu}}^{p} + (\epsilon_0^{\prime})^{p} \leq (\|f^{\prime}-B V_k\|_{p, \hat{\mu}} + \epsilon_0^{\prime})^{p} \leq \|f^{\prime}-B V_k\|_{p, \mu}^{p}$. And thus
\begin{equation}
    \left|\left\|f^{\prime}-B V_k\right\|_{p, {\mu}}^{p}-\left\|f^{\prime}-B V_k\right\|_{p, \hat{\mu}}^{p}\right|>(\epsilon_0^{\prime})^{p}.
\end{equation}
This inequality holds for an analogous reason when $\|f^{\prime}-B V_k\|_{p, \hat{\mu}} > \|f^{\prime}-B V_k\|_{p, \mu}$. And since
\begin{equation}
    \sup _{f \in \mathcal{F}}\left|\|f-B V_k\|_{p, \mu}^{p}-\|f-B V_k\|_{p, \hat{\mu}}^{p}\right| \geq\left|\left\|f^{\prime}-B V_k\right\|_{p, \mu}^{p}-\left\|f^{\prime}-B V_k\right\|_{p, \hat{\mu}}^{p}\right|,
\end{equation}
we can get
\begin{equation}
\label{inequal_to_combine_2}
    \mathbb{P}\left(\sup _{f \in \mathcal{F}}\Big|\|f \!-\! B V_k\|_{p, \mu} \!-\! \|f \!-\! B V_k\|_{p, \hat{\mu}}\Big| \!>\! \epsilon_0^{\prime}\right) \!\leq\! \mathbb{P}\left(\sup _{f \in \mathcal{F}}\left|\|f \!-\! B V_k\|_{p, \mu}^{p} \!-\! \|f \!-\! B V_k\|_{p, \hat{\mu}}^{p}\right| \!>\! \left(\epsilon_0^{\prime}\right)^{p}\right).
\end{equation}
Observe that $\|f-B V_k\|_{p, \mu}^{p}=\mathbb{E}\left[\left|\left(f\left(s_{1}\right)-B V_k\left(s_{1}\right)\right)\right|^{p}\right]$, and $\|f-B V_k\|_{p, \hat{\mu}}^{p}$ is just the sample average approximation of $\|f-B V_k\|_{p, \mu}^{p}$. Using Lemma \ref{lemma:Pollard}, we can get
\begin{equation}
\label{inequal_to_combine_3}
\begin{aligned}
    \mathbb{P}\left(\sup _{f \in \mathcal{F}}\left|\|f \!-\! B V_k\|_{p, \mu}^{p} \!-\! \|f \!-\! B V_k\|_{p, \hat{\mu}}^{p}\right| \!>\! \left(\epsilon_0^{\prime}\right)^{p}\right) \!&\leq\! 8 \mathbb{E} \!\left[\mathcal{N}\Big(\frac{\left(\epsilon_0^{\prime}\right)^{p}}{8}, \mathcal{F}\left(X^{1: N}\right) \!\Big) \!\right] e^{-\frac{N}{2} \!\left(\! \frac{1}{8} \!\big( \frac{\epsilon_0^{\prime}}{2 V_{max }} \!\big)^{p}\right)^{2}}
    \\ &= 8\mathcal{N}_{0}(N) e^{-\frac{N}{2}\left(\frac{1}{8}\big(\frac{\epsilon_0^{\prime}}{2 V_{max }}\big)^{p}\right)^{2}}.
\end{aligned}
\end{equation}
Making the right-hand side upper bounded by $\delta_0^{\prime} = \delta_0/4$ yields a lower bound of $N$, displayed in (\ref{equa:bound for N}). Then (\ref{inequa:Q_epsilon}) can be proved by combining (\ref{inequal_to_combine_1}), (\ref{inequal_to_combine_2}), and (\ref{inequal_to_combine_3}).

Now we turn to prove (\ref{single iter inequality 2}) and (\ref{single iter inequality 4}). For an arbitrary function $f \in \mathcal{F}$, using the triangle inequality, we have
\begin{equation}
    \left|\|f-B V_k\|_{p, \hat{\mu}}-\|f-\hat{V}_k\|_{p, \hat{\mu}}\right| \leq\|B V_k-\hat{V}_k\|_{p, \hat{\mu}}.
\end{equation}

So if we show that $\|B V_k-\hat{V}_k\|_{p, \hat{\mu}} \leq \epsilon_0^{\prime}$ holds w.p. $ 1-\delta_0^{\prime}$, then (\ref{single iter inequality 2}) and (\ref{single iter inequality 4}) can be proved by choosing $f = V_{k+1}$ and $f = f^{*}$, respectively. To prove $\|B V_k-\hat{V}_k\|_{p, \hat{\mu}} \leq \epsilon_0^{\prime}$, Recall that the data $(s,a,\hat{s^{\prime}},\hat{r})$ used to calculate $\hat{V}_k$ is sampled from real environment w.p. $\beta$ (case 1), or from the learned model w.p. $1-\beta$ (case 2).

For case 1, it satisfies that $\hat{s^{\prime}} = s^{\prime}$ and $\hat{r} = r$. So we have
\begin{equation}
    \mathbb{P}\left(|\hat{r} + \gamma V_k(\hat{s^{\prime}}) - r - \gamma V_k(s^{\prime})| > \epsilon_0^{\prime} \right) = 0.
\end{equation}

And for case 2, like in previous literature \cite{pets,slbo}, assume the reward function is known, i.e., $\hat{r} = r$. We have
\begin{align}
    \mathbb{P}\left(|\hat{r} + \gamma V_k(\hat{s^{\prime}}) - r - \gamma V_k(s^{\prime})| > \epsilon_0^{\prime} \right) &= \mathbb{P}\left(
    \gamma| (V_k(\hat{s^{\prime}}) -V_k(s^{\prime}))| > \epsilon_0^{\prime} \right) \nonumber
    \\&\leq \mathbb{P}\left(\gamma K_{V}\|(\hat{s^{\prime}} -s^{\prime})\|_{2} > \epsilon_0^{\prime} \right)
    \nonumber
    \\&= \mathbb{P}\left(\|\hat{s^{\prime}} -s^{\prime}\|_{2} > \frac{\epsilon_0^{\prime}}{\gamma K_{V}} \right).
\end{align}

By using the model error assumption: $f(\|\hat{s^{\prime}}-s^{\prime}\|_2)=\frac{2}{\sqrt{2 \pi} \sigma} \exp \left(-\frac{\|\hat{s^{\prime}}-s^{\prime}\|_2^{2}}{2 \sigma^{2}}\right)$, we can write that:
\begin{align}
    \mathbb{P}\left(|\hat{r} + \gamma V_k(\hat{s^{\prime}}) - r - \gamma V_k(s^{\prime})| > \epsilon_0^{\prime} \right) &\leq \mathbb{P}\left(\|\hat{s^{\prime}} -s^{\prime}\|_{2} > \frac{\epsilon_0^{\prime}}{\gamma K_{V}} \right)\nonumber
    \\&= \mathbb{P}\left(\frac{\|\hat{s^{\prime}} -s^{\prime}\|}{\sigma} > \frac{\epsilon_0^{\prime}}{\gamma K_{V} \sigma} \right)\nonumber
    \\&= 2\Big(1-\Phi(\frac{\epsilon_0^{\prime}}{\gamma K_{V} \sigma})\Big),
\end{align}
where $\Phi$ is the Cumulative Distribution Function (CDF) of the standard normal distribution. Combine these two cases, we can get
\begin{equation}
    \mathbb{P}\left(\Big|\hat{r} + \gamma V_k(\hat{s^{\prime}}) - r - \gamma V_k(s^{\prime})\Big| > \epsilon_0^{\prime} \right) \leq 2(1-\beta)\Big(1-\Phi(\frac{\epsilon_0^{\prime}}{\gamma K_{V} \sigma})\Big).
\end{equation}
Making the right-hand side upper bounded by $\delta_0^{\prime}/(N|\mathcal{A}|)$ yields an upper bound of $\sigma$, displayed in (\ref{equa:bound for sigma}). And for each sampled state $s_i$, since
\begin{equation}
    \nonumber
    \left|B V_k\left(s_{i}\right)-\hat{V}_k\left(s_{i}\right)\right| \leq \max _{a \in \mathcal{A}} \Big|\hat{r} + \gamma V_k(\hat{s^{\prime}}) - r - \gamma V_k(s^{\prime})\Big|,
\end{equation}
by using a union bounding argument, we can get
\begin{equation}
     \mathbb{P}\left(\left|B V_k\left(s_{i}\right)-\hat{V}_k\left(s_{i}\right)\right|>\epsilon_0^{\prime} \right) \leq \delta_0^{\prime} / N.
\end{equation}
And by another union bounding argument, we have
\begin{equation}
     \mathbb{P}\left(\max_{i=1,\ldots,N}\left|B V_k\left(s_{i}\right)-\hat{V}_k\left(s_{i}\right)\right|>\epsilon_0^{\prime} \right) \leq \delta_0^{\prime}.
\end{equation}
And therefore,
\begin{equation}
     \mathbb{P}\left(\frac{1}{N}\left|B V_k\left(s_{i}\right)-\hat{V}_k\left(s_{i}\right)\right|>\epsilon_0^{\prime} \right) \leq \delta_0^{\prime}.
\end{equation}
Hence, we have proved that $\|B V_k-\hat{V}_k\|_{p, \hat{\mu}} \leq \epsilon_0^{\prime}$ holds w.p. at least $ 1-\delta_0^{\prime}$. This completes the whole proof.
\end{proof}

\begin{lemma}
\label{lemma:munos theorem2}
(\citet{munos2008finite}, Theorem 2) Under the concentrability assumption (Assumption \ref{assumption:concentrability}), let $\rho$ be an arbitrary state distribution, and $C_{\rho,\mu}$ be the discounted-average concentrability coefficient. If each iteration error can be bounded as $\left\|V_{k+1}-B V_k\right\|_{p, \mu} \leq d_{p, \mu}(B \mathcal{F}, \mathcal{F})+ (1-\gamma)^{2}\epsilon /(4 \gamma C_{\rho, \mu}^{1 / p})$ w.p. at least $1-\delta/K$ for $0 \leq k < K$, then the loss due to using $\pi_K$ instead of the optimal policy $\pi^{*}$ satisfies that w.p. at least $1-\delta$,
\begin{equation}
    \left\|V^{*}-V^{\pi_{K}}\right\|_{p, \rho} \leq \frac{2 \gamma}{(1-\gamma)^{2}} C_{\rho, \mu}^{1 / p} d_{p, \mu}(B \mathcal{F}, \mathcal{F})+\epsilon,
\end{equation}
provided that 
\begin{equation}
\label{equa: bound for K}
    \gamma^{K}<\left[\frac{(1-\gamma)^{2}}{8 \gamma V_{\max }} \epsilon\right]^{p}.
\end{equation}
\end{lemma}

Lemma \ref{lemma:munos theorem2} shows that if the error in each iteration can be bounded with high probability, then the return discrepancy between $\pi_{K}$ and $\pi^{*}$ can be bounded with high probability when the number of iterations $K$ is large enough. Moreover, we can set $\rho$ to the distribution of the states we care about more, e.g., the initial states we start to execute $\pi_{K}$.

\begin{theorem}
\label{appendix theorem: return bound}
($\beta$-mixture sampling-based FVI bound) Assume that the pseudo-dimension of the value function space is finite as in Proposition \ref{proposition:pseudo-dimension}, under the concentrability assumption (Assumption \ref{assumption:concentrability}) and the same assumptions of Lemma \ref{Lemma:Single Iteration Error Bound}, let $\rho$ be an arbitrary state distribution, $C_{\rho,\mu}$ be the discounted-average concentrability coefficient and $\pi_{K}$ be the greedy policy w.r.t. $V_K$. Let $V^{\pi_{K}}$ and $V^{*}$ be the expected return of executing $\pi_{K}$ and the optimal policy $\pi^{*}$ in real environment, respectively. Define $N_{\rm{real}} = N \cdot |\mathcal{A}|\cdot \beta$ as the expected number of real samples. Then the following bound holds w.p. at least $1-\delta$:
\begin{equation}
\begin{aligned}
    \|V^{*}-V^{\pi_{K}}\|_{p, \rho} \leq& \frac{2 \gamma}{(1-\gamma)^{2}} C_{\rho, \mu}^{1 / p} d_{p, \mu}(B \mathcal{F}, \mathcal{F})
+ O\left(\Big(\frac{\beta |\mathcal{A}|}{N_{\rm{real}}}\big(\log (\frac{N_{\rm{real}}}{\beta |\mathcal{A}|})+\log (\frac{K}{\delta})\big)\Big)^{\frac{1}{2p}}\right) 
\\&+O\Big(\Phi^{-1}\big(1-\frac{\beta\delta}{8KN_{\rm{real}}(1-\beta)}\big) \sigma\Big)
+O\left(\gamma^{K/p} V_{\max }\right).
\end{aligned}    
\end{equation}

\end{theorem}
\begin{proof}
Lemma \ref{lemma:munos theorem2} implies that we can bound $\left\|V^{*}-V^{\pi_{K}}\right\|_{p, \rho}$ w.p. at least $1-\delta$ as
\begin{equation}
\label{equa:final bound}
    \left\|V^{*}-V^{\pi_{K}}\right\|_{p, \rho} \leq \frac{2 \gamma}{(1-\gamma)^{2}} C_{\rho, \mu}^{1 / p} d_{p, \mu}(B \mathcal{F}, \mathcal{F})+\epsilon,
\end{equation}
via bounding $\left\|V_{k+1}-B V_k\right\|_{p, \mu}$ w.p. at least $1-\delta/K$ as
\begin{equation}
\label{equa:single bound}
    \left\|V_{k+1}-B V_k\right\|_{p, \mu} \leq d_{p, \mu}(B \mathcal{F}, \mathcal{F})+ (1-\gamma)^{2}\epsilon /(4 \gamma C_{\rho, \mu}^{1 / p}).
\end{equation}
Using Lemma \ref{appendix Lemma:Single Iteration Error Bound} to bound (\ref{equa:single bound}) by setting $\epsilon_0 = (1-\gamma)^{2}\epsilon /(4 \gamma C_{\rho, \mu}^{1 / p})$ and $\delta_0 = \delta / K$. We can get the corresponding bounds of $N$ and $\sigma$ w.r.t. $\epsilon$. Combined with (\ref{equa: bound for K}), we can write $\epsilon$ as
\begin{equation}
    \epsilon = O\left(\Big(\frac{1}{N}\big(\log (\mathcal{N}_{0}(N))+\log (\frac{K}{\delta})\big)\Big)^{\frac{1}{2p}}\right) 
+O\Big(\Phi^{-1}\big(1-\frac{\delta}{8KN|\mathcal{A}|(1-\beta)}\big) \sigma\Big)
+O\left(\gamma^{K/p} V_{\max }\right).
\end{equation}
Then using Proposition \ref{proposition:pseudo-dimension} to bound $\mathcal{N}_{0}(N)$, we can get:
\begin{equation}
\begin{aligned}
\label{equa:epsilon}
    \epsilon &\leq O\left(\Big(\frac{1}{N}\big(\log (1/\epsilon_0)+\log (\frac{K}{\delta})\big)\Big)^{\frac{1}{2p}}\right) 
+O\Big(\Phi^{-1}\big(1-\frac{\delta}{8KN|\mathcal{A}|(1-\beta)}\big) \sigma\Big)
+O\left(\gamma^{K/p} V_{\max }\right)
\\ &\leq O\left(\Big(\frac{1}{N}\big(\log (N)+\log (\frac{K}{\delta})\big)\Big)^{\frac{1}{2p}}\right) 
+O\Big(\Phi^{-1}\big(1-\frac{\delta}{8KN|\mathcal{A}|(1-\beta)}\big) \sigma\Big)
+O\left(\gamma^{K/p} V_{\max }\right),
\end{aligned}
\end{equation}
where the last inequality comes from inequality \ref{equa:bound for N}. Plugging (\ref{equa:epsilon}) into (\ref{equa:final bound}) and substituting $N$ with $N_{\rm{real}}/(\beta |\mathcal{A}|)$ complete the proof.
\end{proof}

\section{Model Error Assumption}
\label{appendix:model error assumption}
In Lemma \ref{Lemma:Single Iteration Error Bound}
and Theorem \ref{theorem: return bound}, we assume that the model error between the real next state $s^{\prime}$ and the predicted ones $\hat{s^{\prime}}$ obeys a half-normal distribution, i.e., the probability density function $f(\|\hat{s^{\prime}}-s^{\prime}\|_2)=\frac{2}{\sqrt{2 \pi} \sigma} \exp \left(-\frac{\|\hat{s^{\prime}}-s^{\prime}\|_2^{2}}{2 \sigma^{2}}\right)$. In this section, we give an empirical analysis of this model error assumption. To be more specific, we test the trained dynamics model in MBPO on newly collected data (typically 10k transitions) and plot the frequency of different model error ranges in Figure \ref{fig:model error} to approximate the distribution. We can observe that the approximated distribution is close to the half-normal distribution, i.e., the probability of small model error is higher than that of large model error, which empirically shows the rationality of the assumption.
 
 \begin{figure*}[htb]
	\centering
	
	\vspace{-5pt}
	\includegraphics[width=0.99\textwidth]{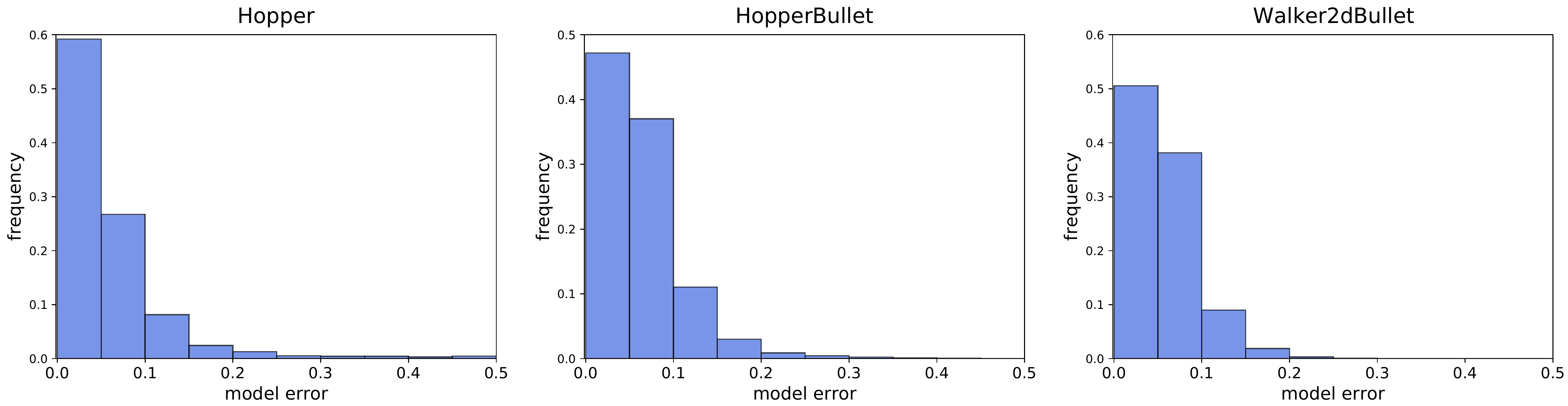}
	\caption{Frequency distribution histogram of the model error $\|\hat{s^{\prime}}-s^{\prime}\|_2$. The y-axis indicates the frequency of the corresponding range.}\label{fig:model error}
\end{figure*}
 
\section{Algorithm Comparison}
\label{appendix:algo comparison}
We compare the hyperparameters that can be scheduled in different Dyna-style MBRL algorithms in Table \ref{tab:algo comparison}. From the comparison, real ratio and rollout length can only be scheduled in MBPO and MA-BDDPG since other algorithms merely use the imaginary data to train the policy and need to generate model rollouts from the initial state to the end of the episode. Considering that MBPO is much more effective than MA-BDDPG in continuous control benchmark tasks \cite{mbpo,kalweit2017uncertainty}, we finally choose MBPO as a representative running case.

\begin{table}[htb]
\centering
\caption{Comparison of the hyperparameters that can be scheduled in different Dyna-style MBRL algorithms.}\label{tab:algo comparison}
\vskip 0.10in
\begin{tabular}{@{}c|c|c|c|c@{}}
\toprule
 & Real Ratio & \tabincell{c}{Policy Training\\Iteration}  & \tabincell{c}{Model Training\\Frequency}  & Rollout Length\\ \midrule
MBPO \cite{mbpo} & \checkmark & \checkmark & \checkmark & \checkmark\\
MA-BDDPG \cite{kalweit2017uncertainty} & \checkmark & \checkmark & \checkmark & \checkmark\\
SLBO \cite{slbo} & $\times$ & \checkmark & \checkmark & $\times$\\
ME-TRPO \cite{metrpo} & $\times$ & \checkmark & \checkmark & $\times$\\
MB-MPO \cite{clavera2018model} & $\times$ & \checkmark & \checkmark & $\times$\\
PAL/MAL \cite{rajeswaran2020game} & $\times$ & \checkmark & \checkmark & $\times$\\
\bottomrule
\end{tabular}
\end{table}
 
\section{More Experimental Results}
\label{appendix:more experiments results}
\subsection{Hyperparameter Importance}
\label{sec:full figure}
\begin{figure*}[htb!]
	\centering
	\vspace{-5pt}
	\includegraphics[width=0.95\textwidth]{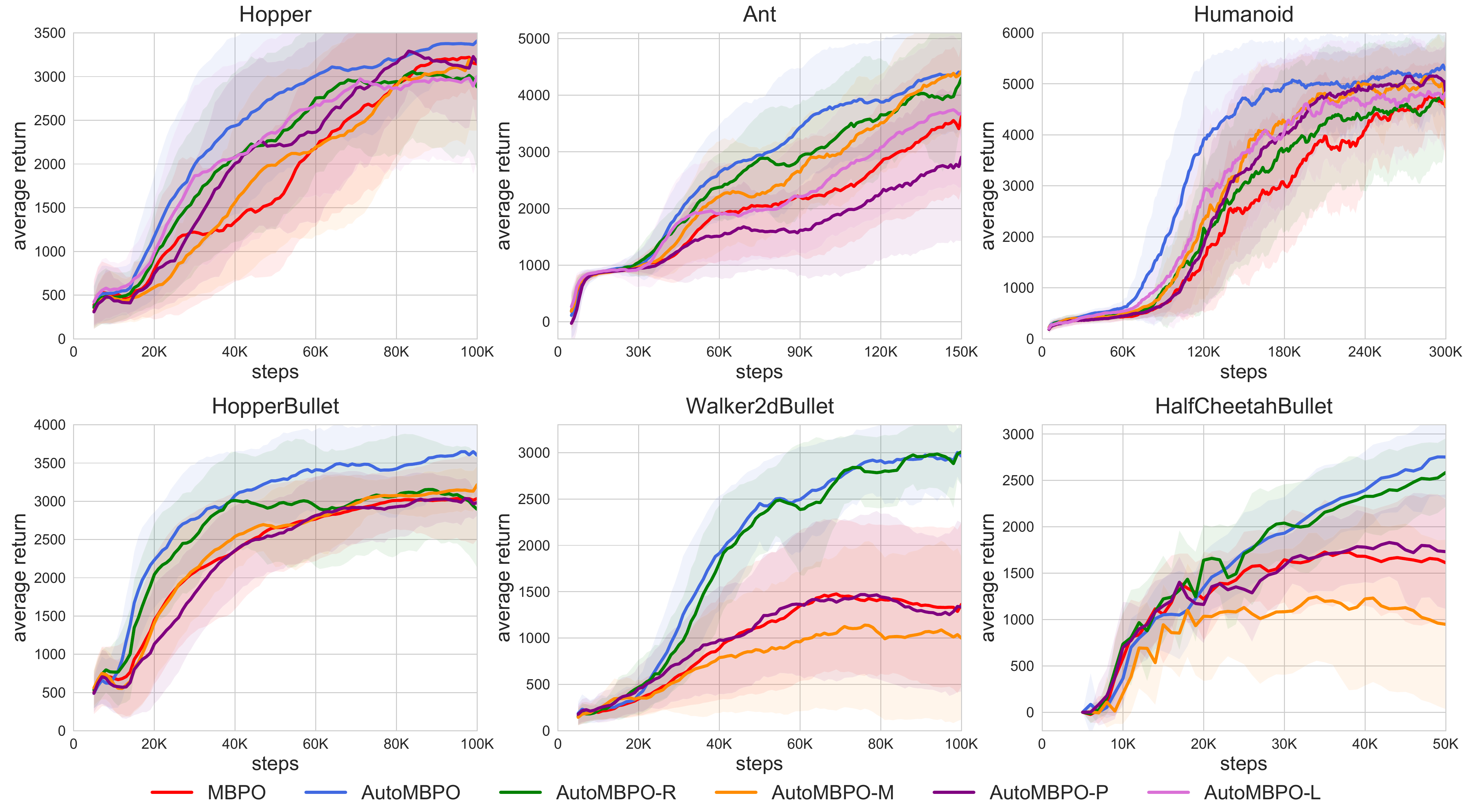}
	\caption{Complete figures of the hyperparameter importance experiments in Section \ref{sec:Importance of Hyperparameters}.
	}\label{fig:hyper_impor_full}
\end{figure*}

Due to the page limit, we only show the results of hyperparameter importance on three environments in Section \ref{sec:Importance of Hyperparameters}. For a more comprehensive analysis, we plot the results on all the six environments in Figure \ref{fig:hyper_impor_full}. The conclusion is the same as discussed in Section \ref{sec:Importance of Hyperparameters}, i.e., using the hyper-controller to schedule the real ratio retains much of the advantage of AutoMBPO.

\subsection{Controller Transfer}

From the results in Figure \ref{fig:hyperparameter}, we observe that the hyperparameter schedules on the three PyBullet environments are similar, especially for the real ratio $\beta$. Then we further conduct experiments to test whether the hyper-controller trained on these three tasks can transfer to others without additional fine-tuning. Results are shown in Figure \ref{fig:transfer}. Though the performance of the transferred hyper-controller is slightly worse than the original ones, they all surpass the MBPO baseline with a considerable margin, which shows the potential of the hyper-controller to capture the commonality to generalize to different tasks.
\begin{figure*}[htb]
	\centering
	\vspace{-4pt}
	\includegraphics[width=1\textwidth]{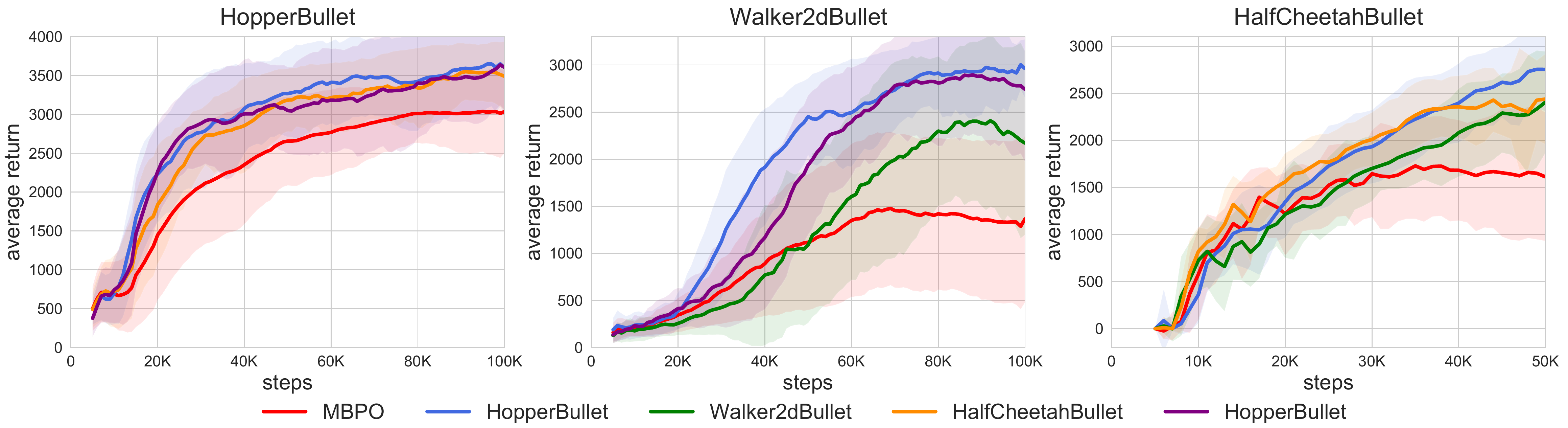}
	\vspace{-15pt}
	\caption{The transfer results of the hyper-controller among PyBullet environments with $A_3^2=6$ transferring cases in total. Each figure's title represents the target task, and each specific line indicates the source task. For example, the green line in the middle figure represents transferring the hyper-controller learned on HopperBullet to Walker2dBullet.}
	
	\vspace{-10pt}
	\label{fig:transfer}
\end{figure*}

\subsection{Effectiveness of AutoMBPO}
\label{appendix:effectiveness}
In Algorithm \ref{alg:AutoMBPO}, we train the hyper-controller until it achieves acceptable performance. One concern is whether the hyper-controller improves monotonically during training. We plot the performance of the hyper-controller in different training phases in Figure \ref{fig:diff_phase}. For example, suppose we train the hyper-controller with 200 hyper-MDP episodes, i.e., 200 MBPO instances in total, then the purple line ($1\%$-$20\%$) represents the average return of the $1$-$40$th MBPO instances, and so on. Monotonic improvement of hyper-controller can be observed during its training phase, which further demonstrates the effectiveness of our algorithm.

Another concern may be the extensive computational requirements for AutoMBPO. According to the computational time table in Appendix \ref{appendix:computing infrastructure}, take Hopper as an example: our algorithm takes about 90 hours to train the hyper-controller, about 5-6 times the time to train a complete MBPO instance, i.e., equivalent to 5-6 trials of MBPO hyperparameters. It is almost impossible to manually find a suitable configuration for these hyperparameters within such few trials.

\begin{figure*}[htb!]
	\centering
	\vspace{-5pt}
	\includegraphics[width=0.95\textwidth]{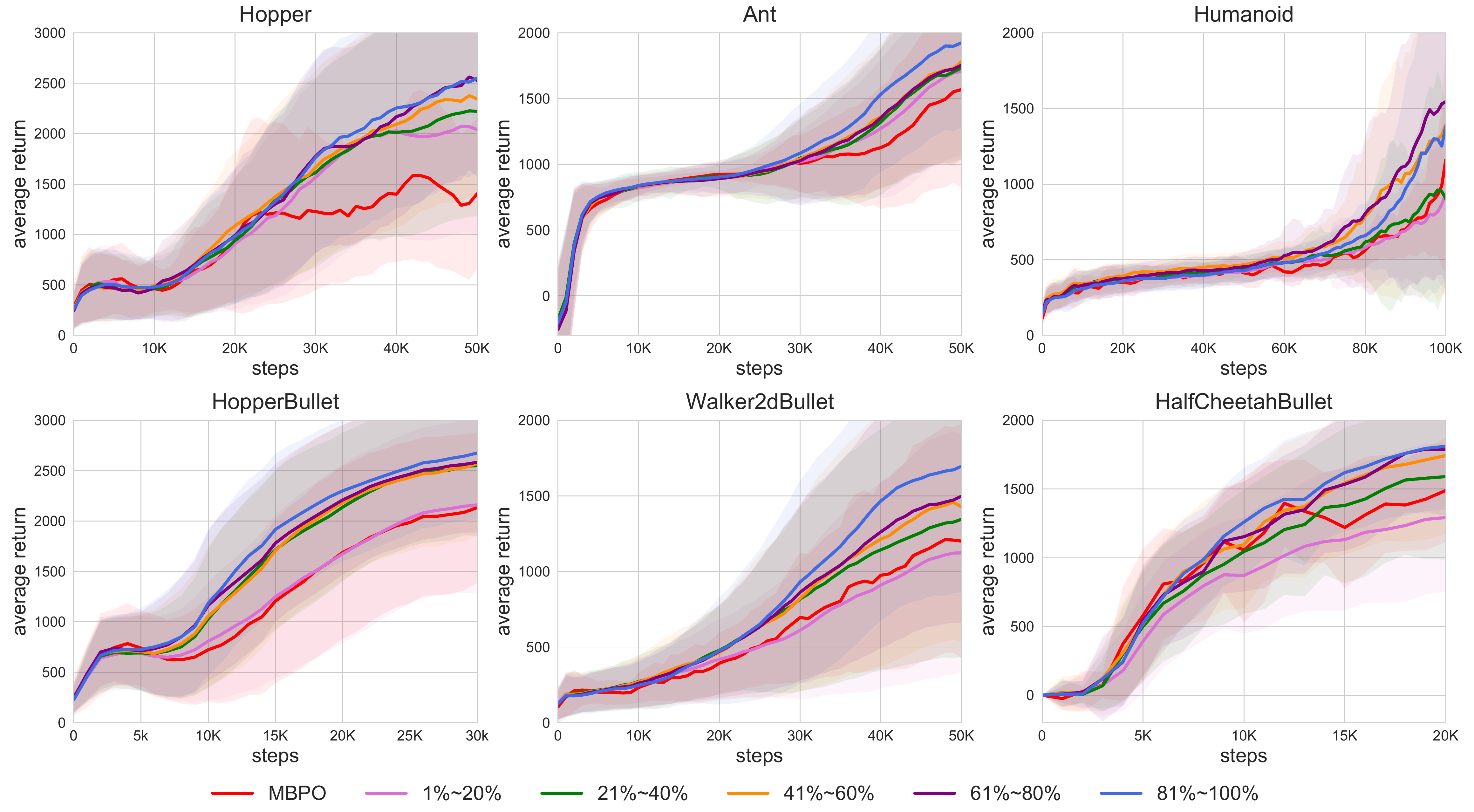}
	\vspace{-5pt}
	\caption{Performance of the hyper-controller in different training phases. Results are the average on six trials of hyper-controller training over different random seeds.
	}\label{fig:diff_phase}
\end{figure*}

\subsection{Hyperparameter Study}

\begin{figure*}[htb!]
	\centering
	\vspace{-5pt}
	\includegraphics[width=0.95\textwidth]{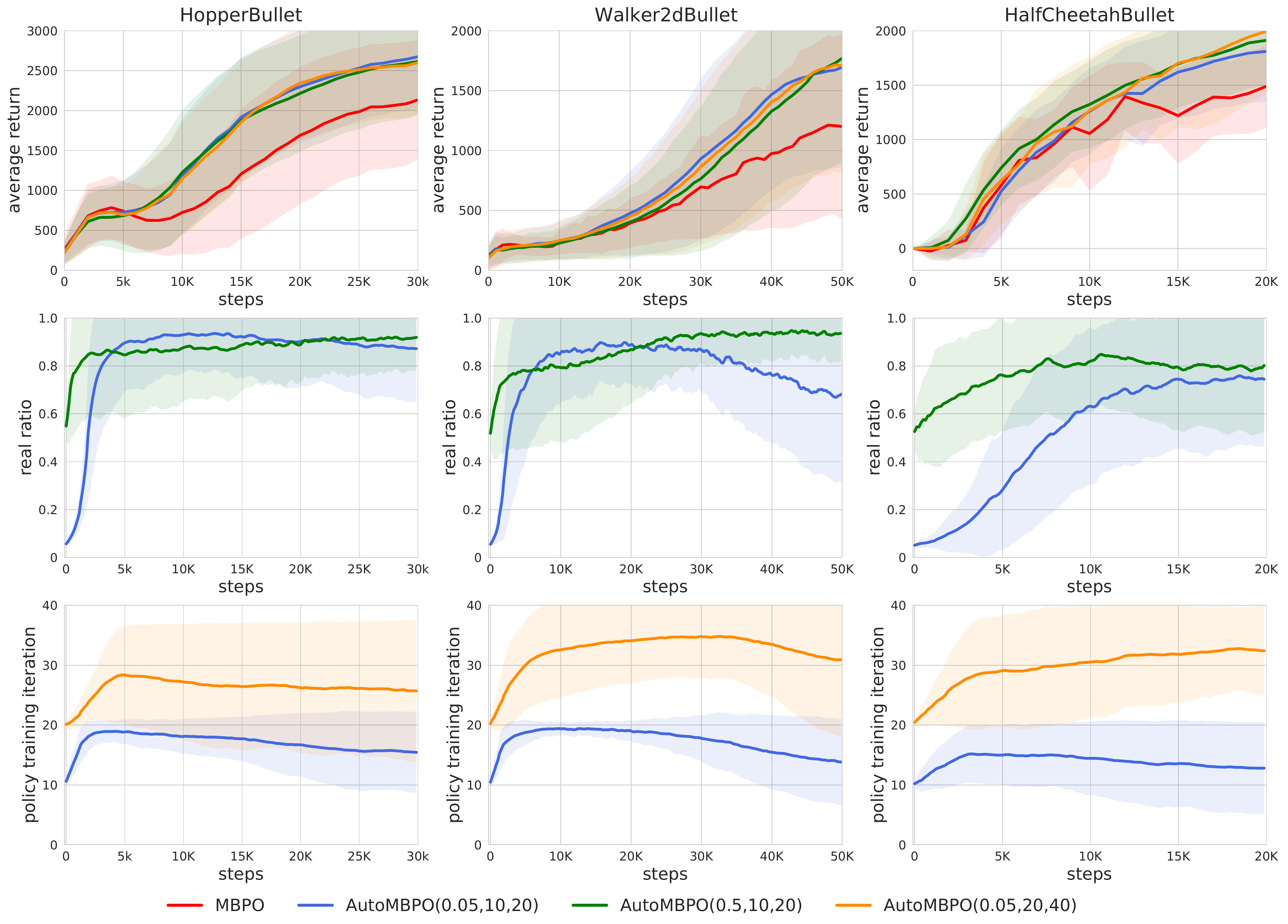}
	\vspace{-5pt}
	\caption{Top: Performance comparison of AutoMBPO(0.5, 10, 20), AutoMBPO(0.05, 20, 40) and the original AutoMBPO(0.05, 10, 20). Middle and bottom: the corresponding schedule of real ratio and policy training iteration.
	}\label{fig:diff_phase_0.5}
\end{figure*}

Since AutoMBPO utilizes the hyper-controller to adjust the hyperparameters of MBPO dynamically, e.g., $\pm1$, the initial values of these hyperparameters become the hyperparameters of AutoMBPO. We also want to investigate the effect of these initial values on the learning process. In the experiments of Section \ref{sec: exp}, we initialize the real ratio to $0.05$, the policy training iteration to $10$, and limit the policy training iteration to $[1,20]$ for computation efficiency. We denote the original one as AutoMBPO(0.05, 10, 20). In this section, we further conduct experiments of AutoMBPO(0.5, 10, 20) and AutoMBPO(0.05, 20, 40) on three PyBullet environments. The performance and the hyperparameter schedules are shown in Figure \ref{fig:diff_phase_0.5}.

From the comparison, we can find that changing the initial value of real ratio or policy training iteration does not influence the final performance much, which shows the robustness of AutoMBPO to the hyperparameter initialization. Notice that the policy training iteration of AutoMBPO(0.05, 20, 40) is much larger than that of AutoMBPO(0.05, 10, 20), while the performance does not improve much. This finding is consistent with our conclusion in Section \ref{sec:hyperparameter visualization}, i.e., increasing the policy training iteration is not always a good choice.

\subsection{State Feature Ablation}
In section \ref{sec:Hyper-MDP Formulation}, inspired by Theorem \ref{theorem: return bound}, we include the number of real samples $N_{\rm{real}}$ and the model loss $\mathcal{L}_{\hat{T}}$ into the state formulation of hyper-MDP. In this section, we ablate these two features and only use other features as the state to train the controller, denoted as AutoMBPO-SA. Results are shown in Figure \ref{fig:state_ablation}. We find that ablating these two features degrades the performance, which highlights the importance of elaborate state design for hyper-controller learning.
\begin{figure*}[htb]
	\centering
	\vspace{-4pt}
	\includegraphics[width=0.85\textwidth]{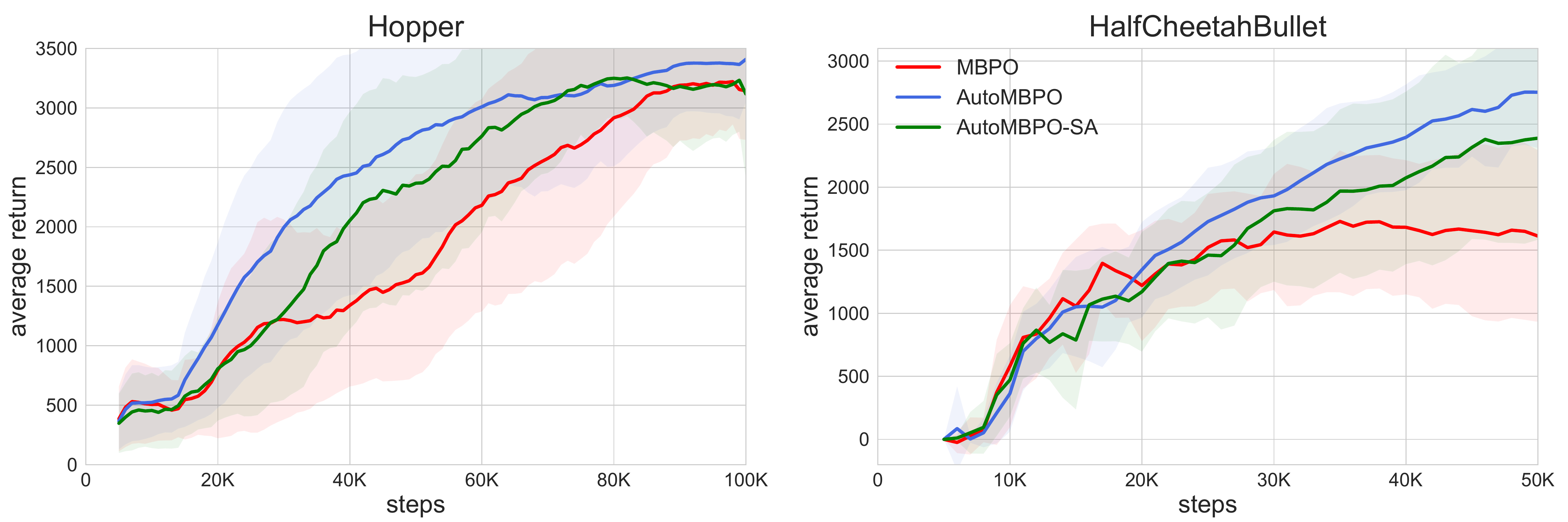}
	\vspace{-5pt}
	\caption{Results of the state feature ablation experiment. AutoMBPO-SA denotes the AutoMBPO variant of excluding real samples number and model loss from the hyper-MDP state.}
	
	\vspace{-10pt}
	\label{fig:state_ablation}
\end{figure*}

\subsection{Statistical Significance}
 Since the shaded areas in Figure \ref{fig:comparison result} overlap, we further add a t-test to the results of the original MBPO and AutoMBPO. We use the average return to perform t-test and list the p-values in Table \ref{tab:p_values}. From the result, the p-values are all less than 0.05, confirming the statistical significance of the results.
 
 \begin{table}[ht]
\centering
\caption{t-test to the average returns of the original MBPO and AutoMBPO.}\label{tab:p_values}
\vskip 0.10in
\begin{tabular}{@{}c|c|c|c|c|c|c@{}}
\toprule
& Hopper & Ant & Humanoid  & \tabincell{c}{Hopper\\Bullet} & \tabincell{c}{Walker2d\\Bullet} & \tabincell{c}{HalfCheetah\\Bullet}\\ \midrule
p-value & 0.0078 & 0.0031 & 6e-5 & 0.0052 & 9e-5 & 0.0276\\
\bottomrule
\end{tabular}
\end{table}
 
\section{Baseline Implementation Details}
\textbf{PBT.} The implementation of PBT mainly follows \citet{zhang2021importance}. Specifically, we train 10 MBPO instances in parallel, each with randomly initialized hyperparameters. After each episode, 20$\%$ of the MBPO instances with low returns are replaced by the top 20$\%$ (both hyperparameters, network parameters, and data buffer) or re-initialized (only hyperparameters) with a certain probability. 

\textbf{RoR.}
For RoR, we utilize the same Hyper-MDP definition as AutoMBPO since the original Hyper-MDP definition in \citet{dong2020intelligent} is not suitable for MBPO. Moreover, we run RoR 6 times, each using the same amount of data as our method. In the result of Figure \ref{fig:comparison result}, we compare to the best performing RoR for a fair comparison.

\section{Experimental Settings}
\label{appendix:experimental settings}
We first provide the details of the experimental environments in Table \ref{tab:environments}. Among them, Hopper, Ant, and Humanoid are the same version used in \citet{mbpo}. For the remaining three PyBullet environments, we modify the reward function to be similar to that of Mujoco to facilitate training since we found that the default reward setting is too complicated for MBPO to solve in preliminary experiments. Then, we present the experimental settings in different environments in Table \ref{tab:experiments}.

\begin{table}[htb!]
\centering
\caption{Environment settings in our experiments. $\theta_t$ denotes the joint angle, $x_t$ denotes the position in x-direction, $a_t$ denotes the action control input, and $z_t$ denotes the height.}\label{tab:environments}
\vskip 0.10in
\begin{tabular}{@{}c|c|c|c|c@{}}
\toprule
& \tabincell{c}{State\\Dimension} & \tabincell{c}{Action\\Dimension} & 
Reward Function & \tabincell{c}{Termination States\\Condition}  \\ \midrule
Hopper & 11 & 3 & $\dot{x}_{t}-0.001\left\|a_{t}\right\|_{2}^{2} + 1$ & $z_{t} \leq 0.7$ or $\theta_{t} \geq 0.2$ \\\midrule
Ant & 27 & 8 & $\dot{x}_{t}-0.5\left\|a_{t}\right\|_{2}^{2} + 1$ & $z_{t} \leq 0.2$ or $z_{t} \geq 1.0$ \\\midrule

Humanoid & 45 & 17 & 0.25$\dot{x}_{t}-0.1\left\|a_{t}\right\|_{2}^{2} + 5$ & $z_{t} \leq 1.0$ or $z_{t} \geq 2.0$ \\\midrule

\tabincell{c}{Hopper\\Bullet} & 15 & 3 & 5$\dot{x}_{t}-0.001\left\|a_{t}\right\|_{2}^{2} + 1$ & $z_{t} \leq 0.8$ or $|\theta_{t}| \geq 1.0$ \\\midrule
\tabincell{c}{Walker2d\\Bullet} & 22 & 6 & 5$\dot{x}_{t}-0.001\left\|a_{t}\right\|_{2}^{2} + 1$ & $z_{t} \leq 0.8$ or $|\theta_{t}| \geq 1.0$ \\\midrule
\tabincell{c}{HalfCheetah\\Bullet} & 26 & 6 & 5$\dot{x}_{t}-0.001\left\|a_{t}\right\|_{2}^{2}$ & None \\
\bottomrule
\end{tabular}
\end{table}

\begin{table}[ht]
\centering
\caption{Experimental settings in different environments. Specifically, a hyper-MDP episode consists of m target-MDP episodes, i.e., the whole training process of an MBPO instance, and a target-MDP episode consists of $H$ timesteps in the environments.} \label{tab:experiments}
\vskip 0.10in
\begin{tabular}{@{}c|c|c|c|c|c|c|c@{}}
\toprule
& &  {Hopper} & {Ant} & {Humanoid} & \tabincell{c}{Hopper\\Bullet} & \tabincell{c}{Walker2d\\Bullet} & \tabincell{c}{HalfCheetah\\Bullet}\\ \midrule
    &  \tabincell{c}{hyper-MDP episodes} & 100 & 100 & 200 & 200 & 200 & 100\\ \midrule
  $m$ &  \tabincell{c}{target-MDP epi-\\sodes for training} & 50 & 50 & 100 & 30 & 50 & 20\\ \midrule
$M$ & \tabincell{c}{target-MDP epi-\\sodes for evaluation} & 100 & 150 & 300 & 100 & 100 & 50\\ \midrule
$H$ & \tabincell{c}{timesteps per \\ target-MDP episode} & \multicolumn{6}{|c}{1000}\\
\bottomrule
\end{tabular}
\end{table}

\section{Computing Infrastructure}
\label{appendix:computing infrastructure}
We present the computing infrastructure and the corresponding computational time used to train the hyper-controller in Table \ref{tab:infra and time}.

\begin{table}[ht]
\centering
\caption{Computing infrastructure and the corresponding computational time.}\label{tab:infra and time}
\vskip 0.10in
\begin{tabular}{@{}c|c|c|c|c|c|c@{}}
\toprule
& Hopper & Ant & Humanoid  & \tabincell{c}{Hopper\\Bullet} & \tabincell{c}{Walker2d\\Bullet} & \tabincell{c}{HalfCheetah\\Bullet}\\ \midrule
CPU& \multicolumn{3}{|c|}{32 cores} & \multicolumn{3}{|c}{16 cores}\\\midrule
GPU& \multicolumn{3}{|c|}{RTX2080TI$\times$2} & \multicolumn{3}{|c}{V100 $\times$2}\\\midrule
\tabincell{c}{computation\\time in hours} & 90.92 & 95.05  & 245.33 & 56.91 & 149.88 & 30.48\\ 
\bottomrule
\end{tabular}
\end{table}

\section{Hyperparameters}
\label{appendix:hyperparameters}
Table \ref{tab:hyper} lists the hyperparameters used in training the hyper-controller. Other hyperparameters of MBPO not scheduled by the hyper-controller are the same as the original one \cite{mbpo}. Note that the hyperparameter $\tau$ in Humanoid varies from that in other environments since the original MBPO configuration of Humanoid is different. The original MBPO trains the model per 1000 real timesteps in Humanoid but per 250 real timesteps in other environments. So we set $\tau$ to half of the original interval to keep the maximum model training frequency two times of the original configuration.
\begin{table}[ht]
\centering
\caption{Hyperparameter settings for hyper-controller.} \label{tab:hyper}
\vskip 0.10in
\begin{tabular}{@{}c|c|c|c|c|c|c|c@{}}
\toprule
& &  {Hopper} & {Ant} & {Humanoid} & \tabincell{c}{Hopper\\Bullet} & \tabincell{c}{Walker2d\\Bullet} & \tabincell{c}{HalfCheetah\\Bullet}\\ \midrule

$\tau$ & \tabincell{c}{real timesteps\\interval per action}& \multicolumn{2}{|c|}{125} & 500 & \multicolumn{3}{c}{125} \\ \midrule

 &\tabincell{c}{policy network \\ architecture} & \multicolumn{6}{|c}{MLP with one hidden layer of size 256} \\\midrule
& learning rate & \multicolumn{6}{|c}{$3 \cdot 10^{-4}$}\\\midrule
& batch size & \multicolumn{6}{|c}{64}\\\midrule
& \tabincell{c}{policy updates per \\ hyper-MDP episode} & \multicolumn{6}{|c}{30}\\ \midrule
& \tabincell{c}{initial value \\ of real ratio} & \multicolumn{6}{|c}{0.05}\\ \midrule
& \tabincell{c}{initial value of poli-\\cy training iteration} & \multicolumn{6}{|c}{10}\\ \midrule
& \tabincell{c}{initial value of \\ rollout length} & \multicolumn{6}{|c}{1}\\ \midrule
$\epsilon$ & \tabincell{c}{PPO clip\\ constant}& \multicolumn{6}{|c}{0.2} \\\midrule
$c$ & \tabincell{c}{real ratio\\change constant} & \multicolumn{6}{|c}{1.2} \\\midrule
 & \tabincell{c}{penalty for each \\ model training} & \multicolumn{6}{|c}{0.1}\\

\bottomrule
\end{tabular}
\end{table}

\end{document}